\documentclass[letterpaper, 11pt]{article}
\usepackage[utf8]{inputenc}
\usepackage[margin=1in]{geometry}
\usepackage{natbib}
\usepackage{amsfonts}
\usepackage{amsmath}
\usepackage{amssymb}
\usepackage{amsthm}
\usepackage{bbm}
\usepackage{xcolor}
\usepackage{graphicx}
\usepackage{subcaption}
\usepackage{float}
\usepackage{hyperref}
\usepackage{enumitem}
\setlist{leftmargin=5.5mm}

\newcommand{\ST}{\mathcal{S}}
\newcommand{\AC}{\mathcal{A}}

\newcommand{\pbm}{\mathcal{P}}

\newcommand{\transprob}{\mathcal{T}}
\newcommand{\rn}{\mathbb{R}}
\newcommand{\PP}{\mathbb{P}}
\newcommand{\E}{\mathbb{E}}

\DeclareMathOperator*{\argmin}{arg\,min}
\DeclareMathOperator*{\argmax}{arg\,max}

\newtheorem{defn}{Definition}
\newtheorem{lem}{Lemma}

\newtheorem{coro}{Corollary}
\newtheorem{thm}{Theorem}
\newtheorem{rmk}{Remark}
\newtheorem{ex}{Example}

\title{Identifiability in inverse reinforcement learning}

\author{Haoyang Cao\footnote{Alan Turing Institute, \url{hcao@turing.ac.uk}}, Samuel N. Cohen\footnote{Mathematical Institute, University of Oxford and Alan Turing Institute, \url{samuel.cohen@maths.ox.ac.uk}}, {\L}ukasz Szpruch\footnote{School of Mathematics, University of Edinburgh and Alan Turing Institute, \url{L.Szpruch@ed.ac.uk}}}

\begin{document}

\maketitle

\begin{abstract}
  Inverse reinforcement learning attempts to reconstruct the reward function in a Markov decision problem, using observations of agent actions. As already observed in \citet{russell1998learning} the problem is ill-posed, and the reward function is not identifiable, even under the presence of perfect information about optimal behavior. We provide a resolution to this non-identifiability for problems with entropy regularization. For a given environment, we fully characterize the reward functions leading to a given policy and demonstrate that, given demonstrations of actions for the same reward under two distinct discount factors, or under sufficiently different environments, the unobserved reward can be recovered up to a constant. We also give general necessary and sufficient conditions for reconstruction of time-homogeneous rewards on finite horizons, and for action-independent rewards, generalizing recent results of \cite{Kim2021} and \cite{fu2017learning}.
\end{abstract}

\section{Introduction}\label{sec: review}

Inverse reinforcement learning aims to use observations of agents' actions to determine their reward function.  The problem has roots in the very early stages of optimal control theory; \citet{Kalman64} raised the question of whether, by observation of optimal policies, one can recover coefficients of a quadratic cost function (see also \citet{boyd1994linear}). This question naturally generalizes to the generic framework of Markov decision process and stochastic control. 

In the 1970s, these questions were taken up within economics, as a way of determining utility functions from observations. For instance, \citet{keeney_raiffa_1993} set out to determine a proper ordering of all possible states which are deterministic functions of actions. In this setup, the problem is static and the outcome of an action is immediate. Later in \citet{sargent1978estimation}, a dynamic version of a utility assessment problem was studied, under the context of finding the proper wage through observing dynamic labor demand. 

As exemplified by Lucas' critique\footnote{The critique is best summarized by the quotation: ``Given that the structure of an econometric model consists of optimal decision rules of economic agents, and that optimal decision rules vary systematically with changes in the structure of series relevant to the decision maker, it follows that any change in [regulatory] policy will systematically alter the structure of econometric models.'' (\citet{Lucas1976})}, in many applications it is not enough to find \emph{some} pattern of rewards corresponding to observed policies; instead we may need to identify \emph{the specific} rewards agents face, as it is only with this information that we can make valid predictions for their actions in a changed environment. In other words, we do not simply wish to learn a reward which allows us to imitate agents in the current environment, but which allows us to predict their actions in other settings.

In this paper, we give a precise characterization of the range of rewards which yield a particular policy for an entropy regularized Markov decision problem. This separates the main task of estimation (of the optimal policy from observed actions) from the inverse problem (of inferring rewards from a given policy). We find that even with perfect knowledge of the optimal policy, the corresponding rewards are not fully identifiable; nevertheless, the space of consistent rewards is parameterized by the value function of the control problem. In other words, the reward can be fully determined given the optimal policy and the value function, but the optimal policy gives us no direct information about the value function.

We further show that, given knowledge of the optimal policy under two different discount rates, or sufficiently different transition laws, we can uniquely identify the rewards (up to a constant shift).  We also give conditions under which action-independent rewards, or time-homogenous rewards over finite horizons, can be identified. This demonstrates the fundamental challenge of inverse reinforcement learning, which is to disentangle immediate rewards from future rewards (as captured through preferences over future states).

\section{Background on reinforcement learning}\label{sec:MDP}
The motivation behind inverse reinforcement learning is to use observed agent behavior to identify the rewards motivating agents.  Given these rewards, one can forecast future behavior, possibly under a different environment. In a typical reinforcement learning\footnote{Reinforcement learning and optimal control are closely related problems, where reinforcement learning typically focuses on the challenge of numerically learning a good control policy, while optimal control focuses on the description of the optimizer. In the context of the inverse problem we consider they are effectively equivalent and we will use the terms interchangeably.} (RL) problem, an agent learns an optimal policy to maximize her total reward by interacting with the environment. 

In order to analyse the inverse reinforcement learning problem, we begin with an overview of the `primal' problem, that is, how to determine optimal policies given rewards.  We particularly highlight a entropy regularized version of the Markov decision process (MDP), which provides a better-posed setting for inverse reinforcement learning. For mathematical simplicity, we focus on discrete-time problems with finitely many states and actions; our results can largely be transferred to continuous settings, with fundamentally the same proofs, however some technical care is needed.

\subsection{Discrete Markov decision processes with entropy regularization}
\paragraph{The environment.} We consider a simple Markov decision process (MDP) on an infinite horizon. The MDP $\mathcal{M}=(\ST,\AC,\transprob,f,\gamma)$ is described by: a finite state space $\ST$; a finite set of actions $\AC$; a (Markov) transition kernel  $\transprob:\ST\times \AC \to \pbm(\ST)$, that is, a function $\transprob$ such that $\transprob(s, a)$ gives probabilities\footnote{Here, and elsewhere, we write $\pbm(X)$ for the set of all probability distributions on a set $X$.} of each value of $S_{t+1}$, given the state $S_t=s$ and action $A_t=a$ at time $t$;
    and a reward function $f:\ST\times \AC\rightarrow\mathbb{R}$ with discount factor $\gamma\in[0,1)$.

An agent aims to choose a sequence of actions $\{A_0, A_1,...\}$ from $\AC$ in order to to maximize the expected value of total reward
\[\textstyle \sum_{t=0}^\infty \gamma^t f(S_t, A_t).\]

It will prove convenient for us to allow randomized policies $\pi$, that is, functions $\pi:\ST\to \pbm(\AC)$, where $\pi(\cdot|s)$ is the distribution of actions the agent takes when in state $s$. For a given randomized policy $\pi$, we define $\transprob_{\pi}\in \pbm(\ST)$, the distribution of $S_{t+1}$ given state $S_t$, by
\[\textstyle   \transprob_{\pi}(S_{t+1}=s'|S_t=s)=\sum_{a\in\AC}\transprob(s'|s,a)\pi(a|s).\]

Given an initial distribution $\rho\in \mathcal{P}(\ST)$ for $S_0$ and a policy $\pi:\ST\to \mathcal{P}(\AC)$, we obtain  a (unique) probability measure $\PP^{\pi}_{\rho}$  such that, and any $a\in \AC, s, s'\in \ST$, $\PP^{\pi}_{\rho}(S_0 = s)=\rho(s)$ and
\[ \textstyle 
 \PP^{\pi}_{\rho}(A_t = a|S_t = s)=\pi(a|s), \quad \PP^{\pi}_{\rho}(S_{t+1}=s'| S_t = s, A_t = a)=\transprob(s'|s,a), \text{ for all }t.
\]
We write $\mathbb{E}^{\pi}_{\rho}$ for the corresponding expectation and $\mathbb{E}^{\pi}_{s}$ when the initial state is given by $s\in\ST$. The classic objective in a MDP is to maximize the expected value $\mathbb{E}_{s}^{\pi}\big[\sum_{t=0}^{\infty}\gamma^{t} f(S_t,A_t)\big]$.
With this objective, one can show (for example, see \citet{bertsekas2004stochastic,puterman2014markov}) that there is an optimal deterministic control (i.e. a policy $\pi$, taking values zero and one, which maximizes the expected value). This implies that, typically, an optimal agent will only make use of a single action for each state, and the choice of this action will not vary smoothly with changes in the reward, discount rate, or transition kernel. 

\paragraph{Entropy regularised MDP.} Given the lack of smoothness in the classical MDP, and to encourage exploration, a well-known variation on the classic MDP introduces a regularization term based on the Shannon entropy. Given a policy $\pi$ and regularization coefficient $\lambda\ge 0$, the entropy regularized value of a policy $\pi$, when starting in state $s$, is defined by
\begin{equation}
\label{eq value function}
	V^\pi_{\lambda}(s)
	:= \mathbb{E}_{s}^{\pi}\bigg[\sum_{t=0}^{\infty}\gamma^{t}\bigg( f(s_t,a_t)-\lambda\,\log \Big(\pi(a_t|s_t)\Big)\bigg)\bigg] 
	= \mathbb{E}_{s}^{\pi}\bigg[\sum_{t=0}^{\infty}\gamma^{t} \bigg(f(s_t,a_t)+\lambda \mathcal{H}\big(\pi(\cdot|s_t)\big)\Big)\bigg)\bigg].\,
\end{equation}
Here $\mathcal{H}(\pi) = -\sum_{a\in\AC} \pi(a)\log(\pi(a)) $ is the entropy of $\pi$. We call this setting the regularised MDP $\mathcal{M_{\lambda}}=(\ST,\AC,\transprob,f,\gamma,\lambda)$.
The optimal value is given by $V^{*}_{\lambda}(s):= \max_{\pi} V^\pi_\lambda(s),$
where the maximum is taken over all (randomized feedback\footnote{Given the Markov structure there is no loss of generality when restricting to policies of feedback form. Further, by replacing $\AC$ with the set of maps $\ST\to \AC$ if necessary, all feedback controls $a(s)$ can be written as deterministic controls $a(\cdot)$ in a larger space, so when convenient we can consider controls which do not depend on the state without loss of generality.}) policies $\pi:\ST\to\pbm(\AC)$. 

We define the state-action value of $\pi$ at
$(s,a)\in \ST \times \AC$ by
\begin{equation}
\label{eq:Qpi}
Q^{\pi}_{\lambda}(s,a)=f(s,a)+\gamma\sum_{s'\in \ST}\transprob(s'|s,a)V^\pi_{\lambda}(s').
\end{equation}
The dynamic programming principle (e.g.~\citet[Theorem 2]{haarnoja2017reinforcement}) gives
\begin{equation} \label{eq dp}
\begin{split}
V^{*}_\lambda(s) &= \max_{m \in \pbm(\AC)}\bigg[\sum_{a\in\AC} \bigg(f(s,a) - \lambda\log \Big(m(a)\Big) + \gamma \mathbb{E}_{s_1\sim \transprob(\cdot|s,a)}\Big[V_\lambda^{*}(s_1) \Big]
\bigg)m(a) \bigg]\\
&=\lambda \max_{m \in \pbm(\AC)}\bigg[\frac1\lambda \sum_{a\in\AC} \bigg(f(s,a) 
+ \gamma \mathbb{E}_{s_1\sim \transprob(\cdot|s,a)}\Big[V_\lambda^{*}(s_1) \Big]
\bigg)m(a) +\mathcal{H}(m)\bigg].
\end{split}
\end{equation}
Observing that on the right hand side we are maximizing over a linear function in $m$ plus an entropy term, and applying~\cite[Proposition 1.4.2]{dupuis2011weak}, we have that for any $s\in S$
\begin{equation}
\label{eq softmax}
V^*_\lambda(s) = V^{\pi^\ast_\lambda}_\lambda(s) = \lambda \log \sum_{a\in\AC} e^{ \frac1\lambda \big(   f(s,a) + \gamma \mathbb{E}_{s_1\sim \transprob(\cdot|s,a)}[V_\lambda^{\pi^\ast_\lambda}(s_1) ] \big)}
= \lambda \log \sum_{a\in\AC} \exp\bigg( \frac1\lambda Q^{\pi^\ast_\lambda}_\lambda(s,a) \bigg) \,,
\end{equation}
and the maximum in~\eqref{eq dp} is achieved by the randomized policy $m(a) = \pi^*_\lambda(a|s)$, where 
\begin{equation}
    \label{eq:optimalpolicyMDP}
\pi^\ast_\lambda(a|s) = \frac{\exp\big( \frac1\lambda Q^{\pi^\ast_\lambda}_\lambda(s,a) \big)}{\sum_{a'\in\AC} \exp\big( \frac1\lambda Q_\lambda^{\pi^{\ast}_\lambda} (s,a') \big)}\qquad \text{for }a\in \AC.
\end{equation}
From~\eqref{eq softmax} we see that $\exp\big( V^{\ast}_\lambda(s)/\lambda\big) = \sum_{a\in\AC} \exp\big(Q^{\pi^\ast_{\lambda}}_\lambda(s,a)/\lambda\big)
$ and so we can write the optimal policy as
\begin{align}
	\label{eq:optimalpolicyMDP2}
	\begin{split}\pi^{\ast}_{\lambda}(a|s)&=\exp\Big(\big(Q^{\pi^{\ast}_\lambda}_{\lambda}(s,a)-V^{\ast}_{\lambda}(s)\big)\big/\lambda\Big) \\&= \exp\Big(\big(f(s,a)+ \mathbb{E}_{s_1\sim \transprob(\cdot|s,a)}[\gamma V_\lambda^{\ast}(s_1)-V_\lambda^{\ast}(s) ]\big)\big/\lambda\Big), 
	\end{split}
\end{align}

From this analysis, we make the following observations regarding the regularized MDP:
\begin{itemize}
\item The optimal policy will select all actions in $\mathcal{A}$ with some positive probabilities. 
\item If $\lambda$ is increased, this has the effect of `flattening out' the choice of actions, as seen in the softmax function in \eqref{eq:optimalpolicyMDP}. Conversely, sending $\lambda\to 0$ will result in a true maximizer being chosen, and the regularized problem degenerates to the classical MDP. 
\item Adding a constant to the reward does not change the policy.
\end{itemize}
\begin{rmk}In many modern approaches, one replaces dependence on the state with dependence on a space of `features'. This has benefits when fitting a model, but does not significantly change the problem considered.
\end{rmk}

\section{Analysis of inverse reinforcement learning}\label{sec:IRL}
We now shift our focus to `inverse' reinforcement learning, that is, the problem of inferring the reward function given observation of agents' actions. 

Consider a discrete time, finite-state and finite-action MDP $\mathcal{M}_\lambda$, as described in Section \ref{sec:MDP}. Suppose a `demonstrator' agent acts optimally, and hence generates a \emph{trajectory} of states and actions for the system $\tau = (s_1, a_1, s_2, a_2,...)$.
We assume that it is possible for us to observe $\tau$ (over a long period), and seek to infer the reward $f$ which the agent faces.

A first observation is that, assuming each state $s\in \ST$ appears infinitely often in the sequence $\tau$, and the agent uses a randomized feedback control $\pi_\lambda(a|s)$, it is possible to infer this control. A simple consistent estimator for the control is 
\[(\pi_\lambda)_N(a|s) = \frac{\#\{a_t = a \text{ and }s_t = s; \quad t\leq N\}}{\#\{s_t=s; \quad t\leq N\}} \to \pi_\lambda(a|s) \text{ a.s. as }N\to\infty.\]
Similarly, assuming each state-action pair $(s,a)$ appears infinitely often in $\tau$, we can infer the controlled transition probabilities $\transprob(s'|s,a)$. A simple consistent estimator is given by 
\[\transprob_N(s'|s,a)  = \frac{\#\{s_t=s, a_t = a \text{ and }s_{t+1} = s'; \quad t\leq N\}}{\#\{s_t=s \text{ and } a_t = a ; \quad t\leq N\}} \to \transprob(s'|s,a) \text{ a.s. as }N\to\infty.\]

If our agent is known to follow a regularized optimal strategy, as in \eqref{eq:optimalpolicyMDP2}, and we have a simple accessibility condition\footnote{In particular, for every  pair of states $s, s'$, there needs to exist a finite sequence $s=s_1, s_2, ..., s_n=s'$ of states and $a_1,..., a_{n-1}$ of actions such that $\prod_{k=1}^{n-1} \transprob(s_{k+1}|s_k,a_k)>0$. This is certainly the case, for example, if we assume $\transprob(s'|s,a)>0$ for all  $s, s'\in \ST$ and $a\in\AC$.} on the underlying states, then every state-action pair will occur infinitely often in the resulting trajectory. Therefore, given sufficiently long observations, we will know the values of $\pi(a|s)$ and $\transprob(s'|s,a)$ for all $s, s' \in \ST$ and $a\in \AC$.

This leads, naturally, to an abstract version of the inverse reinforcement learning problem: Given knowledge of $\pi(a|s)$ and $\transprob(s'|s,a)$ for all $s,s'\in \ST$ and $a\in \AC$, and assuming $\pi$ is generated by an agent following an entropy-regularized MDP $\mathcal M_{\lambda}$, can we determine the initial reward function $f$ that the agent faces?

As observed by \citet{Kalman64}, for an unregularized controller the only thing we can say is that the observed controls are maximizers of the state-action value function, and not even that these maximizers are unique. Therefore, very little can be said about the underlying reward in the unregularized setting. Indeed, as already observed in  \citet{russell1998learning} the problem of constructing a reward  using state-action data is fundamentally ill-posed. One pathological case is to simply take $f$ constant, so all actions are optimal. Alternatively, if we infer a unique optimal action $a^{\star}(s)$ for each $s$, we then could take any $f(s,a^{\star}(s)) \in (0,\infty]$ and $f(s,a)=0$ for $a\neq a^{\star}(s)$.

\paragraph{Further literature.} One of the earliest discussions of inverse reinforcement learning (IRL) in the context of machine learning can be found in~\citet{ng2000algorithms}. Their method is to first identify  a class of reward functions, for an IRL problem with finitely many states and actions, a deterministic optimal strategy, and the assumption that the reward function depends only on the state variable. Then, assuming the reward function is expressable in terms of some known basis functions in the state, a linear programming formulation for the IRL problem is presented, to pick the reward function that maximally differentiates optimal policy from the other policies. This characterization of reward functions demonstrates the general non-uniqueness of solutions to IRL problems.

In past two decades, there have been many algorithms proposed to tackle IRL problems. One significant category of algorithms (MaxEntIRL) arises from the maximum entropy approach to optimal control. In~\citet{ziebart2010modeling}, IRL problems were linked with maximum causal entropy problems with statistical matching constraints. Similar models can be found in~\citet{abbeel2004apprenticeship}; \cite{ziebart2008maximum}, \cite{levine2011nonlinear} and~\citet{boularias2011relative}. A connection between maximum entropy IRL and GANs has been established in \cite{finn2016connection}. Further related papers will be discussed in the text below.

In MaxEntIRL, one assumes that trajectories are generated\footnote{As discussed by \citet{levine2018reinforcement}, for deterministic problems this simplifies to $P(\tau)\propto \exp(\sum_t f(a_t,s_t))$, which is often taken as a starting point.} with a law
\[
\mathbb P(\tau) = \frac{\rho(s_0)}Z\prod_{t}\transprob(s_{t+1}|s_t,a_t)e^{\sum_t f(a_t,s_t)}\]
for a constant $Z>0$. Comparing with the distribution of trajectories from an optimal regularized agent, this approach implicitly assumes that $\pi(a|s)\propto \exp\{f(a,s)\}$. Comparing with \eqref{eq:optimalpolicyMDP2}, this is analogous to assuming the value function is a constant (from which we can compute $Z$) and $\lambda=1$. This has a concrete interpretation: that many IRL methods make the tacit assumption that the demonstrator agent is myopic. As we shall see in Theorem \ref{thm:IRL-MDP1}, for inverse RL the value function can be chosen arbitrarily, demonstrating the consistency of this approach with our entropy-regularized agents. We discuss connections with MaxEntIRL further in Appendix \ref{sec:GCL}.

\subsection{Inverse Markov decision problems}\label{sec:IRLMDP}

We consider a Markov decision problem as in Section \ref{sec:MDP}. As discussed above, we assume that we have full knowledge of $\ST, \AC, \transprob,\gamma$, and of the regularization parameter $\lambda$ and the entropy-regularized optimal control $\pi_\lambda$ in \eqref{eq:optimalpolicyMDP2}, but not the reward function $f$.
 
Our first theorem characterizes the set of all reward functions $f$ which generate a given control policy.

\begin{thm}\label{thm:IRL-MDP1}
For a fixed policy $\bar \pi(a|s)>0$, discount factor $\gamma\in [0,1)$, and an arbitrary choice of function $v:\ST\to \rn$, there is a unique corresponding reward function 
\[\textstyle f(s,a) = \lambda\log \bar \pi(a|s) - \gamma\sum_{s'\in \ST}\transprob(s'|s,a)v(s')+v(s)\]
such that the MDP with reward $f$ yields a value function $V_\lambda^{\pi^*_\lambda} = v$ and entropy-regularized optimal policy $\pi^*_\lambda = \bar \pi$.
\end{thm}
\begin{proof}
Fix $f$ as in the statement of the theorem. Then \eqref{eq softmax} gives the corresponding value function
\begin{align*}
V^*_\lambda(s) &= \lambda \log \sum_{a\in \AC} \exp\bigg(\frac{1}{\lambda}\Big(f(s,a)+\gamma\sum_{s'\in \ST}\transprob(s'|s,a)V^*_\lambda(s')\Big)\bigg)\\
&= v(s)+\lambda \log \sum_{a\in \AC}\bar\pi(a|s)\exp\bigg(  \frac{\gamma}{\lambda}\Big(\sum_{s'\in \ST}\transprob(s'|s,a)(V^*_\lambda(s')-v(s'))\Big)\bigg),
\end{align*}
which rearranges to give
\begin{equation}\label{eq:jensenprecursor}\exp(g(s)) = \sum_{a\in \AC}\bar\pi(a|s)\exp\bigg(\gamma   \sum_{s'\in \ST}\transprob(s'|s,a)g(s')\bigg)\end{equation}
 with $g(s) = (V^*_\lambda(s)-v(s))/\lambda$. 
Applying Jensen's inequality, we can see that, for $\underline s \in \argmin_{s\in \ST} g(s)$, 
\[\exp\Big( \min_s g(s)\Big)=\exp\Big(g(\underline s)\Big) \geq \exp\bigg(\gamma \sum_{a\in \AC, s'\in \ST}\bar\pi(a|\underline s)\transprob(s'|\underline s,a)g(s')\bigg).\]
However, the sum on the right is a weighted average of the values of $g$, so 
\[ \sum_{a\in \AC, s'\in \ST}\bar\pi(a|\underline s)\transprob(s'|\underline s,a)g(s') \geq \min_s g(s).\]
Combining these inequalities, along with the fact $\gamma<1$, we conclude that $g(s)\geq 0$ for all $s\in \ST$.

Again applying Jensen's inequality to \eqref{eq:jensenprecursor}, for $\bar s \in \argmax_{s\in\ST}g(s)$ we have
\[\max_s \Big\{\exp\Big(g(s)\Big)\Big\} = \exp\Big(g(\bar s)\Big) \leq \sum_{a\in \AC, s'\in \ST}\bar\pi(a|\bar s)\transprob(s'|\bar s,a)\exp\bigg(\gamma g(s')\bigg).\]
As the sum on the right is a weighted average, we know
\[\sum_{a\in \AC, s'\in \ST}\bar\pi(a|\bar s)\transprob(s'|\bar s,a)\exp\bigg(\gamma g(s')\bigg) \leq \max_s \Big\{\exp\Big(\gamma g(s)\Big)\Big\}.\]
Hence, as $\gamma<1$, we conclude that $g(s)\leq 0$ for all $s\in \ST$.

Combining these results, we conclude that $g\equiv 0$, that is, $V^*_\lambda = v$. Finally, we substitute the definition of $f$ and the value function $v$ into \eqref{eq:optimalpolicyMDP2} to see that the entropy-regularized optimal policy is $\pi^*_\lambda = \bar \pi$.
\end{proof}

As a consequence of this theorem, we observe that the value function is not determined by the observed optimal policy, but can be chosen arbitrarily. We also see that the space of reward functions $f$ consistent with a given policy can be parameterized by the set of value functions.

\begin{rmk}
A simple degrees-of-freedom argument gives this result intuitively. There are $n=|\ST|$ possible states and $k=|\AC|$ possible actions in each state, so the reward function can be described by a vector in $\rn^{n\times k}$ . From the policy, which satisfies $\sum_{a\in\AC}\pi(a|s) = 1$ for all $s$, we observe $n\times (k-1)$ linearly independent values. Therefore, the space of consistent rewards has $n\times k-n\times (k-1) = n$ free variables, which we identify with the $n$ values $\{v(s)\}_{s\in \ST}$.
\end{rmk}
\begin{rmk}
\label{rem:ngremark}
\citet{Ng99policyinvariance} provides a useful insight to our result. In \citet{Ng99policyinvariance} it is assumed that the rewards are of the form $F( S_t, A_t, S_{t+1})$; for a fixed MDP, this adds no generality, as we  can write $f(s,a) = \mathbb{E}[F(s, a, S_{t+1})|S_t=s, A_t=a]$. \citet{Ng99policyinvariance} show that, for any `shaping potential' $\Upsilon:\ST\to \rn$, the reward $\tilde F = F + \gamma \Upsilon(S_{t+1}) - \Upsilon(S_t)$ yields the same optimal policies for \emph{every} (unregularized) MDP. However, shaping potentials do not describe the space of all rewards corresponding to a given policy, for fixed transition dynamics. In our results, we instead parameterize a family of costs $f$ in terms of the value function (Theorem \ref{thm:IRL-MDP1}), and show these are the only costs which lead to the given optimal policy \emph{for a fixed (regularized) MDP}.
\end{rmk}
Given Theorem \ref{thm:IRL-MDP1}, we see that it is not possible to fully identify the reward faced by a single agent, given only observations of their policy. Fundamentally, the issue is that the state-action value function $Q$ combines both immediate rewards $f$ with preferences $v$ over the future state. If we provide data which allows us to disentangle these two effects, for example by considering agents with different discount rates or transition functions, then the true reward can be determined up to a constant, as shown by our next result. In order to clearly state the result, we give the following definition.

\begin{defn}\label{defn:valuedistinguishing}
Consider a pair of Markov decision problems on the same state and action spaces, but with respective discount rates $\gamma, \tilde\gamma$ and transition probabilities $\transprob, \tilde\transprob$. We say that this pair is \emph{value-distinguishing} if, for functions $w, \tilde w:\ST\to \rn$, the statement
\begin{equation}\label{eq:valuedistinguishing}
w(s) - \gamma\sum_{s'\in \ST}\transprob(s'|s,a)w(s')  = \tilde w(s) - \tilde \gamma\sum_{s'\in \ST}\tilde\transprob(s'|s,a)\tilde w(s') \text{ for all }a\in \AC, s\in \ST
\end{equation}
implies at least one of $w$ and $\tilde w$ is a constant function.
\end{defn}

In this definition, note that constant functions $w, \tilde w$ are always solutions to \eqref{eq:valuedistinguishing}, in particular for $c\in \mathbb{R}$ we can set  $w \equiv c$ and $\tilde w \equiv ({1-\gamma}) c/({1-\tilde \gamma})$. However, this is  a system of $|\AC|\times |\ST|$ equations in $2\times|\ST|$ unknowns, so the definition will hold provided our agents' actions have sufficiently varied impact on the resulting transition probabilities. In a linear-quadratic context, it is always enough to vary the discount rates (see Corollary \ref{cor:LQdiscount}).

\begin{thm}\label{thm:IRL-MDP2}
Suppose we observe the policies of two agents solving entropy-regularized MDPs, who face the same reward function, but whose discount rates or transition probabilities vary, such that their MDPs are value-distinguishing. Then the reward function consistent with both agents' actions either does not exist, or is identified up to addition of a constant.
\end{thm}
\begin{proof} 
From Theorem \ref{thm:IRL-MDP1}, if we can determine the value function for one of our agents, then the reward is uniquely identified. Given we know both agents' policies ($\pi$, $\tilde\pi$) and our agents are optimizing their respective MDPs,
 for every $a\in \AC, s\in \ST$, we know the value of
\begin{equation}\label{eq:twoagentValue}
    \lambda \log \frac{\pi(a|s)}{\tilde\pi(a|s)} =  \gamma\sum_{s'\in \ST}\transprob(s'|s,a)v(s') - \tilde\gamma\sum_{s'\in \ST}\tilde\transprob(s'|s,a)v(s') - (v(s) - \tilde v(s))
\end{equation}
where $v, \tilde v$ are the agents' respective value functions.
This is an inhomogeneous system of linear equations in $\{v(s), \tilde v(s)\}_{s\in\ST}$. Therefore, by standard linear algebra (in particular, the Fredholm alternative), it is uniquely determined up to the addition of solutions to the homogeneous equation
 \[0=\gamma\sum_{s'\in \ST}\transprob(s'|s,a)v(s') - \tilde\gamma\sum_{s'\in \ST}\tilde\transprob(s'|s,a)v(s') - (v(s) - \tilde v(s)) \text{ for all }s\in \ST, a\in \AC.\]
 However, as we have assumed our pair of MDPs is value-distinguishing, the only solutions to this equation have at least one of $v$ and $\tilde v$ constant (we assume $v$ without loss of generality). Therefore, the space of solutions to \eqref{eq:twoagentValue} is either empty (in which case no consistent reward exists), or determines $v$ up to the addition of a constant.  Given $v$ is determined up to a constant we can use Theorem \ref{thm:IRL-MDP1} to determine $f$, again up to the addition of a constant.
\end{proof}
Given the addition of a constant to $f$ does not affect the resulting policy (it simply increases the value function by a corresponding quantity), we cannot expect to do better than Theorem \ref{thm:IRL-MDP2} without direct observation of the agent's rewards or value function in at least one state.

\begin{rmk}
Definition \ref{defn:valuedistinguishing} is essentially a statement regarding invertibility of a linear system of equations for $w,\tilde w$. This indicates that the stability of the result of Theorem \ref{thm:IRL-MDP2} is principally determined by whether this linear system is well conditioned, as can be measured by the ratio of its largest to second smallest singular values (the second smallest is due to the constant functions always being in the kernel of the system) not being too large. Given the inevitable error arising from statistical estimation of policies and transition functions, a well conditioned system is often a key requirement in practice. A similar observation will also be valid for the uniqueness results in later sections.
\end{rmk}

\begin{rmk}
 Our results show that it is typically sufficient to observe an MDP under \emph{two} environments (transitions and discount factors) in order to identify the reward. 
 This can be contrasted with \cite{amin2016towards} and \cite{amin2017repeated} who show that, if the demonstrator is observed in multiple (suitably chosen) environments, the (state-only) reward can be identified up to a scaling and shift (the scaling is natural, given they do not use an entropy regularization). \cite{ratliff2006} consider a finite number of environments, but explicitly do not attempt to estimate the `true' underlying reward.
\end{rmk}

\section{Finite horizon results}\label{sec:finitehorizon}

Over finite horizons, for general costs, similar results hold to those already seen on infinite horizons. An entropy-regularized optimizing agent will use a policy $\pi^*=\{\pi^*_t\}_{t=0}^{T-1}$ which solves the following problem with terminal reward $g$ and (possibly time-dependent) running reward $f$:
\[\begin{aligned}
\max_{\pi}&\E^{\pi}_s\bigg[\sum_{t=0}^{T-1}\gamma^t\Big(f(t, s_t^\tau, a_t^\tau)-\lambda\log\pi_t(a^\tau_t|s^\tau_t)\Big)+\gamma^Tg(s^\tau_T)\bigg].\end{aligned}\]
For any $\pi=\{\pi_t\}_{t=0}^{T-1}$, $s\in\ST$, $a\in\AC$, and $t\in \{0,\dots, T-1\}$ write
\begin{align*}
Q_t^\pi(s,a)&=f(t,s,a)+\gamma\E_{S'\sim\transprob(\cdot|s,a)}\big[V^\pi_{t+1}(S')\big],\\
V^\pi_t(s)&=\E_{A\sim\pi_t(\cdot|s)}\big[Q^\pi_t(s,A)-\lambda\log\pi_t(A|s)\big], \qquad V^\pi_T(s)=g(s)
\end{align*}
Then, similarly to the infinite-horizon discounted case discussed in the main text, we have $V^*_T=g$ and for $t\in\{0,\dots,T-1\}$,
\begin{align*}
Q^*_{t}(s,a)&=f(t, s,a)+\gamma\E_{S'\sim\transprob(\cdot|s,a)}\Big[V^*_{t+1}(S')\Big],\\
V^*_t(s)&=V^{\pi^*}_t(s)=\lambda\log\sum_{a'\in\AC}\exp\Big\{Q^*_t(s,a')/\lambda\Big\},\\
\pi_t^*(a|s)&=\exp\Big\{Q^*_t(s,a)/\lambda\Big\}\bigg/\sum_{a'\in\AC}\exp\Big\{Q^*_t(s,a')/\lambda\Big\}=\exp\left\{\bigg(Q^*_t(s,a)-V^*_t(s)\bigg)/\lambda\right\}.
\end{align*}

Rearranging this system of equations, for any chosen function $v:\{0,...,T\}\times \ST \to \mathbb{R}$ with $v(T,\cdot) = g(\cdot)$, we see that $\pi^*_t(a|s)$ is the optimal strategy for the reward function
\begin{align*}
    f(t,s,a)= \lambda \log \pi_t^*(a|s) - \gamma \sum_{s'\in \ST}\transprob(s'|s,a)v(t+1,s') + v(t,s),
\end{align*}
in which case the corresponding value function is $V^*=v$. In other words, the identifiability issue discussed earlier remains. We note that identifying $\pi$ in this setting is more delicate than in the infinite-horizon case, as it is necessary to observe many finite-horizon state-action trajectories, rather than a single infinite-horizon trajectory.

\subsection{Time-homogeneous finite-horizon identifiability.} Following the release of a first preprint version of this paper, \cite{Kim2021} was published and presents a closely related analysis, for entropy-regularized deterministic MDPs with zero terminal value. We here give an extension of their result which covers the stochastic case and includes an arbitrary (known) terminal reward.

The key structural assumptions made by \cite{Kim2021} are that the reward is time-homogeneous (that is, $f$ does not depend on $t$), and that there is a finite horizon. As discussed in the previous section, there is no guarantee that an arbitrary observed policy will be consistent with these assumptions (that is, whether there exists any $f$ generating the observed policy). However, given a policy consistent with these assumptions, and mild assumptions on the structure of the MDP, we shall see that unique identification of $f$ is possible up to a constant.

Before describing our findings, we first present the following lemma\footnote{Thanks to Victor Flynn for discussion on the formulation and proof of this result.} from elementary number theory, which will prove useful in what follows.
\begin{lem}\label{lem:numbertheory}
 Let $\mathcal{R}\subset \mathbb{N}$ be a set of natural numbers, with the property that $\mathcal{R}$ is closed under addition (if $a,b\in \mathcal{R}$ then $a+b\in \mathcal{R}$). Suppose $\mathcal{R}$ has greatest common divisor $1$ (i.e.~$\mathrm{gcd}(\mathcal{R})=1$). Then there exist elements $a,b\in \mathcal{R}$ which are coprime (i.e.~$\mathrm{gcd}(a,b)=1$). Furthermore, for any coprime $a,b\in \mathcal{R}$, for all $c\geq ab$, we know $c\in \mathcal{R}$, in particular, there exist at least two distinct pairs of nonnegative integers $\lambda,\mu$ such that $\lambda a + \mu b = c$. 
\end{lem}
\begin{proof}
We first show a coprime pair $a,b\in \mathcal{R}$ exists. As $\mathrm{gcd}(\mathcal{R}\cap\{x:x\leq y\})$ is decreasing in $y$, and the integers are discrete, there exists a smallest value $y$ such that $\mathrm{gcd}(\mathcal{R}\cap\{x:x\leq y\}) =1$. Applying B\'ezout's lemma, there exist integers $\{\lambda_k\}_{k\leq y}$ such that 
\[\sum_{k\in\mathcal{R}, k\leq y}\lambda_k k=1.\]
Rearranging this sum by taking all negative terms to the right hand side, we obtain the desired positive integers $a = \sum_{\{k\in\mathcal{R}, k\leq y, \lambda_k>0\}}\lambda_k k$ and $b = \sum_{\{k\in\mathcal{R}, k\leq y, \lambda_k<0\}}|\lambda_k| k$ which satisfy $a=b+1$ (so $a$ and $b$ are coprime) and $a,b\in \mathcal{R}$ (as $\mathcal{R}$ is closed under addition).

We now take an arbitrary coprime pair $a,b\in \mathcal{R}$. Again by B\'ezout's lemma, there exist (possibly negative) integers $\tilde \lambda, \tilde\mu$ such that $\tilde \lambda a + \tilde\mu b = 1$, and hence $\tilde \lambda c a + \tilde \mu c b = c$. However, for any integer $k$ it follows that $(\tilde \lambda c +kb) a + (\tilde \mu c-ka) b = c$. Since this holds for all $k\in \mathbb{Z}$, we can choose $k$ such that $1\leq \lambda = (\tilde \lambda c +kb) \leq b$. However, this implies that $(\tilde \lambda c +kb) a\leq ab \leq c$, and so $\mu = (\tilde \mu c-ka) \geq 0$. As $\mathcal{R}$ is closed under addition, we see that $c = \lambda a + \mu b\in \mathcal{R}$.

To see non uniqueness, we simply observe that if $c\geq ab$, in the construction above we have $\mu\geq a$, and hence $(\lambda+b, \mu-a)$ is an alternative pair of coefficients.
\end{proof}

We now present the first assumption on the structure of the MDP.
\begin{defn}
We say a MDP has full access at horizon $T$ (from a state $s$) if, for some distribution over actions, for all states $s'$ we have $\mathbb{P}(S_{T-1}=s'|S_0=s)>0$. 
\end{defn}
It is easy to verify that this definition does not depend on the choice of distribution over actions (provided it has full support).

This is slightly weaker than assuming that the Markov chain underlying the MDP (with random actions) is irreducible, as there may exist transient states from which we have full access. It is a classical result (commonly stated as a corollary to the Perron--Frobenius theorem) that an irreducible aperiodic Markov chain has full access (from every state). \citet{Kim2021} give an alternative graph-theoretic view, based on  the closely related notion of $T$-coverings.

\begin{thm}\label{thm:timehomog-necessity}
Consider an MDP with unknown time-homogeneous reward function $f$. In order for $f$ to be identified (up to a global constant) from observation of optimal policies and the resulting transitions up to some horizon $T>0$, with initialization from some state $s$, it is necessary that the MDP has full access at some horizon $T'\geq T$ (from state $s$).
\end{thm}
\begin{proof} 

We suppose that $f$ can be identified, and  first show that all states can be accessed from $s$, that is, for each $s'$ there exists $T>0$ such that $\mathbb{P}^\pi(S_T=s'|S_0=s)>0$, but that $T$ can vary with $s'$. Suppose, for contradiction, there are states which cannot be reached by a path starting in $s$. It is clear that it is impossible to identify the cost associated with any state which cannot be accessed, as we obtain no information about actions in these states.

It remains to show that, if $f$ can be identified, we can reach all states using paths of a common length. We initially focus on the paths from $s$ to $s$. If we can return in precisely $T$ steps, then (by the Markov property) we can also return in $kT$ steps, for any $k\in \mathbb{N}$. Therefore either the set  $\{T:\mathbb{P}^\pi(S_T=s|S_0=s)>0\}$ is unbounded, or the state $s$ will never be revisited (in the language of Markov chains, it is ephemeral), and in particular will never be visited by a path starting in any other state. Therefore, it is clear that we can add a constant to its rewards independently of all other states' rewards, as this will not affect decision making -- we leave this state immediately and never return. Therefore the reward cannot be determined up to a global constant.

Next, still focusing on paths from $s$ to $s$,  we show that $T$ can take \emph{any} value above some bound. Let $\bar t$ be the greatest common divisor of $\mathcal{R}=\{t:\mathbb{P}^\pi(S_t=s|S_0=s)>0\}$. For contradiction,  suppose $\bar t>1$. Then our system is periodic, and by classical results on irreducible matrices (e.g. \cite[Theorem 1.3]{Seneta2006}) we know that there is a partition of $\mathcal{S}$ into $\bar t$ sets, such that we will certainly make transitions within the states $\mathcal{S}_0\to \mathcal{S}_1 \to ...\to \mathcal{S}_{\bar t-1}\to\mathcal{S}_0\ni s$. By adding $c\in \mathbb{R}$ to the rewards of states in $\mathcal{S}_0$, and subtracting $c/\gamma$ from rewards of states in $\mathcal{S}_1$, we do not affect behavior. Therefore the reward cannot be identified up to a global constant unless $\bar t = 1$. 

However, if $\bar t = 1$ then,as the set $\mathcal{R}$ is closed under addition (by concatenating cycles), Lemma \ref{lem:numbertheory} then implies that  $\mathcal{R}$ must contain all sufficiently large values, that is, it is possible to return to the initial state in any sufficiently large number of steps.

Finally, we have seen that it is possible to transition from $s$ to $s'$ in a finite number of steps, and that it is possible to transition from $s$ to $s$ in any sufficiently large number of steps. From the Markov property we conclude that for every value of $T'$ sufficently large, for all choices of $s'$ we have $\mathbb{P}^\pi(S_{T'}=s'|S_{0}=s)>0$.
\end{proof}

The following definition is most easily expressed by associating our finite state space $\ST$ with the set of basis vectors $\{e_k\}_{k=1}^N\subset \mathbb{R}^N$, and writing the transitions $\transprob(s'|s,a)$, for $a\in\AC$ in terms of the transition matrix  $\mathbb{T}(a)$ with \[\mathbb{T}(a)_{ij} = \transprob(e_j|e_i, a).\]
While this definition is quite abstract, we will see that it precisely describes when many IRL problem with fixed terminal reward can be solved.
\begin{defn}\label{def:fullactionrank}
We say an $N$-state MDP has full action-rank on horizon $T$, starting at a state $s\equiv e_i$, if the matrix with rows given by
\[ \bigg\{e_i^\top\bigg(\sum_{t=0}^{T-1} \gamma^t\prod_{t'=0}^{t-1}\mathbb{T}(a_{t'})\bigg); \qquad a_0,...,a_t \in \AC\bigg\}
\]
is of rank $N$ (with the convention that products are taken sequentially on the right, that is $\prod_{t=0}^2 A_t = A_0 A_1 A_2$, and the empty product is the identity).
\end{defn}

\begin{rmk}\label{rmk:occupationdensity}
Observe that $e_i^\top\prod_{t'=0}^{t-1}\mathbb{T}(a_{t'})$ is the expected state of $S_t$ given $S_0=e_i$, when following the actions $\{a_0,..., a_{t}\}$. Hence, the quantity $e_i^\top\big(\sum_{t=0}^{T-1} \gamma^t\prod_{t'=0}^{t-1}\mathbb{T}(a_{t'})\big)$ is a time-weighted expected occupation density for the process, that is, a measurement of how long we spend in each state. We have full action-rank if our actions are sufficiently varied that there are $N$ linearly independent such density vectors (cf. \cite[Corollary 1]{Kim2021}, where it is the state--action occupation density which is considered).
\end{rmk}

\begin{thm}\label{thm:finitehorizonidentification}
Suppose our MDP has full action rank and full access, at horizon $T$, from an initial state $s_0$. Then the time-homogeneous IRL problem is well posed, that is, knowledge of the (time-dependent) entropy-regularized optimal strategy $\pi^*_t(a|s)$, and the terminal reward $g$, is sufficient to uniquely determine a time-homogeneous running reward $f$, if it exists, up to a constant.

Conversely, if our MDP has full access but not full action rank at horizon $T$, from the state $s_0$, the IRL problem remains ill posed.
\end{thm}
\begin{proof} 
We first prove the sufficiency statement. The optimal policy satisfies
\[\lambda\log \pi^*_{t}(a|s) = Q^*_{t}(s,a)-V^*_{t}(s) = f(s,a)+\gamma\Big(\sum_{s'}\transprob(s'|s,a) V^*_{t+1}(s')\Big)-V^*_t(s).\]
We write (for notational simplicity), $\upsilon(s)=V^*_{T-1}(s)$, and hence, given $V_T^*\equiv g$ by assumption,
\begin{equation}\label{valueactionrecursion2}
f(s,a) = \upsilon(s)+\lambda\log \pi^*_{T-1}(a|s) - \gamma\Big(\sum_{s'}\transprob(s'|s,a) g(s')\Big).
\end{equation}
This shows that $f$ is completely determined (if it exists) by the  function $\upsilon$.

We also observe that for every $t$ we have the recurrence relation
\begin{align*}
    V^*_t(s) &= -\lambda\log \pi^*_{t}(a|s)  +f(s,a)+\gamma\Big(\sum_{s'}\transprob(s'|s,a) V^*_{t+1}(s')\Big)\\
     &= \lambda\log \frac{\pi^*_{T-1}(a|s)}{\pi^*_{t}(a|s)} + \upsilon(s) +\gamma\Big(\sum_{s'}\transprob(s'|s,a) \big(V^*_{t+1}(s')-g(s')\big)\Big).
\end{align*}
This holds for any choice of action $a$ (unlike the usual dynamic programming relation, which only involves the optimal policy). Writing $\mathbf{V}_t$ for the vector with components $\{V^*_t(s)\}_{s\in \mathcal{S}}$ we have the recurrence relation
\begin{equation}\label{valueactionrecursion}
    \mathbf{V}_t = \Upsilon_t(a) + \upsilon + \gamma \mathbb{T}(a)\mathbf{V}_{t+1}; \qquad \mathbf{V}_{T-1}=\upsilon,
\end{equation}
where $\Upsilon_t$ is a known vector valued function, with components
\[[\Upsilon_t(a)]_s = \lambda\log \frac{\pi^*_{T-1}(a|s)}{\pi^*_{t}(a|s)} - \gamma\sum_{s'}\transprob(s'|s,a)g(s').\]
Solving the recurrence relation, we have, for any sequence of actions $a_0,..., a_{T-1}$ (with the convention that the empty matrix product is the identity)
\[\mathbf{V}_0 = \bigg(\sum_{t=0}^{T-1}\Big[\gamma^t\Big(\prod_{t'=0}^{t-1}\mathbb{T}(a_{t'})\Big)\Upsilon_t(a_t)\Big]\bigg) + \bigg(\sum_{t=0}^{T-1} \gamma^t\prod_{t'=0}^{t-1}\mathbb{T}(a_{t'})\bigg)\upsilon + \gamma^{T}\Big(\prod_{t'=0}^{T-1}\mathbb{T}(a_{t'})\Big)g.\]

From this linear system, we can extract the single row corresponding to the fixed initial state $s_0$. Assuming this is the row indicated by the $e_i$ basis vector, we have
\begin{equation}\label{eq:Vinitialstructure}V_0^*(s_0) = e_i^\top\bigg(\sum_{t=0}^{T-1} \gamma^t\prod_{t'=0}^{t-1}\mathbb{T}(a_{t'})\bigg)\upsilon + G(a_0,...,a_{T-1})\end{equation}
for a known function $G$, expressible in terms of $\gamma$, $g$ and $\{\pi^*_t\}_{t=0}^{T-1}$. 

 Now that the MDP has full action-rank, the system of equations,
\[-G(a_0,\dots,a_{T-1})=e_i^\top\bigg[\sum_{t=0}^{T-1}\gamma^{t}\prod_{t'=0}^{t-1}\mathbb{T}(a_{t'})\bigg]v,\quad\forall a_0,\dots,a_{T-1},\]
admits at most one solution, denoted by $\bar\upsilon$. Substituting into \eqref{eq:Vinitialstructure}, we have a unique solution to the equation $V_0^*(s_0) = 0$. However, we need to consider all possible values of $V_0^*(s_0)$.

For any choice of actions $\{a_t\}_{t=0}^{T-1}$, \[e_i^\top\bigg[\sum_{t=0}^{T-1}\gamma^{t}\prod_{t'=0}^{t-1}\mathbb{T}(a_{t'})\bigg]\mathbf{1}=\begin{cases}
\frac{1-\gamma^T}{1-\gamma},&\gamma\in(0,1),\\
T,&\gamma=1.
\end{cases}\]
Here $\mathbf{1}$ denotes the all-one vector in $\mathbb{R}^{N}$. Therefore, the set of all possible $(V_0^*(s_0),\upsilon)$ pairs is given by
\[\begin{cases}
\Big\{(c,\bar\upsilon+\frac{c(1-\gamma)}{1-\gamma^T}):\forall c\in\mathbb{R}\Big\},&\gamma\in(0,1),\\
\Big\{(c,\bar\upsilon+\frac{c}{T}):\forall c\in\mathbb{R}\Big\},&\gamma=1.
\end{cases}\]
From \eqref{valueactionrecursion2}, we conclude that $f$ can be identified up to a constant.

To show necessity, we observe from the above that, if the system is not full action-rank, then there exists a linear subspace of choices of $\upsilon$, which do not differ only by constants, such that we can construct the same value vectors $\mathbf{V}_t$ for all $t$, satisfying \eqref{valueactionrecursion} and hence \eqref{valueactionrecursion2}.  It follows that we have a nontrivial manifold of rewards $f$ which generate the same optimal policies, that is, the rewards are not identifiable.
\end{proof}

As a corollary, we demonstrate a generalized version of \cite[Theorem 2]{Kim2021}.
\begin{coro}\label{cor:DeterministicFiniteIdentifiable}
Suppose $\gamma\neq0$ and our MDP is deterministic, that is $\transprob(s'|s,a) \in\{0,1\}$, and one of the following holds:
\begin{enumerate}[label = (\roman*)]
    \item \label{caseaperiodic} the underlying Markov chain is irreducible and aperiodic (i.e. with randomly chosen actions, the underlying Markov chain is irreducible and aperiodic)
    \item \label{caseselfloop}the initial state $s_0=e_i$ admits a self-loop (i.e. it is possible to transition from this state to itself), and all states can be accessed from the initial state in at most $d$ transitions
    \item  \label{casecycles}there exist cycles\footnote{A cycle is a sequence of possible transitions which start and end in the same state. The length of a cycle is defined to be the number of transitions, e.g. a cycle $\{s_0\to s_1 \to s_2 \to s_0\}$ has length $3$. An irreducible Markov chain is aperiodic if there is no common factor (greater than one) of the lengths of all cycles.} starting at the initial state $s_0=e_i$  with lengths $R,R'$, such that $\mathrm{gcd}(R,R')=1$, and all states can be accessed from the initial state in at most $d$ transitions.
\end{enumerate} Then there exists a horizon $T$ such that the time-homogeneous IRL problem is well posed (as in Theorem \ref{thm:finitehorizonidentification}). In particular, in case \ref{caseselfloop}, it is sufficient to take any finite $T\geq d+1$; in case \ref{casecycles} it is sufficient to take any finite $T\geq d+RR'$.
\end{coro}
\begin{proof} 
We first observe that it is a classical result on Markov chains (see, for example, \cite[Theorem 1.5]{Seneta2006}) that the conditions of case \ref{caseaperiodic} guarantee those of case \ref{casecycles}, for some choice of $R,R'>0$. The conditions of case \ref{caseselfloop} also guarantee those of case \ref{casecycles}, with both the cycles being the self-loop. It is therefore sufficient to consider case \ref{casecycles}.

To show that the MDP has full action rank, we observe that for every possible path of states, there exists a corresponding sequence of actions, and vice versa. We will therefore use these different perspectives interchangeably. We also observe that, as our MDP is deterministic, $e_i^\top\prod_{t'=0}^{t-1}\mathbb{T}(a_{t'})$ is a vector indicating the current state at time $t$, when started in state $e_i$. Therefore, \[\mathbb{O}_{T-1}(\{a_t\}_{t\ge 0}) := e_i^\top\bigg(\sum_{t=0}^{T-1} \gamma^t\prod_{t'=0}^{t-1}\mathbb{T}(a_{t'})\bigg)\]
is a row vector, containing a time-weighted occupation density -- in particular, if $\gamma=1$, it simply counts the number of times we have entered each state. (This is in contrast to Remark \ref{rmk:occupationdensity}, where we have an \emph{expected} occupation density; here we can simplify given the control problem is deterministic.) Our aim, therefore, is to construct a collection of paths which give a full-rank system of occupation densities.

Starting in state $s_0\equiv e_i$, consider a shortest path (i.e. a path with the fewest number of transitions) to each state $s'$. Denote these paths $r_{s'} = \{s_0\to ... \to s'\}$, and the corresponding sequence of actions $a^{s'}$. These paths have lengths $|r_s|$ and time-weighted occupation densities $\mathbb{O}_{|r_s|}(\{a_t^{s}\}_{t\ge 0})$ which are linearly independent (a longer path will contain states not in a shorter path, while paths of the same length will differ in their final state; by reordering the states we can then obtain a lower-triangular structure in the matrix of occupation densities $[\mathbb{O}_{|r_s|}(\{a_t^s\}_{t\ge 0})]_{\{a_t\}\subset\AC}$). This gives us $N=|\ST|$ paths, of varying lengths, with linearly independent occupation densities.

We now consider prefixing our paths with cycles, in order to make them the same length. Fix an arbitrary integer value $T'\geq\max_s|r_s|+|Q||Q'|-1$. By Lemma \ref{lem:numbertheory}, for all states $s$, there exist nonnegative integers $\lambda_s, \mu_s$ such that $T'=\lambda_s|Q|+\mu_s|Q'|+|r_s|$. Therefore, taking the concatenated path consisting of $\lambda_s$ repeats of cycle $Q$, then $\mu_s$ repeats of cycle $Q'$, then our shortest path $r_s$, gives us a path from $s_0$ to $s$ of length $T'$. Denote each of these paths $P_s$.

Concatenation of paths has an elegant effect on the occupation densities: If $Q$ is a cycle and $r$ a path (starting from the terminal state of $Q$), their concatenation $Q*r$ and corresponding actions $a^Q,a^{r}, a^{Q*r}$, then the occupation densities combine linearly:
\begin{equation}\label{eq:concatenationoccupation}\mathbb{O}_{|Q*r|}(\{a_t^{Q*r}\}_{t\ge 0}) = \mathbb{O}_{|Q|-1}(\{a_t^{Q}\}_{t\ge 0}) + \gamma^{|Q|}\mathbb{O}_{|r|}(\{a_t^{r}\}_{t\ge 0}),\end{equation}
(observe that the occupation density excludes the (repeated) final state of the cycle).

We now observe that for the initial state, the shortest path is of length zero (i.e. has no transitions). From Lemma \ref{lem:numbertheory}, as $T'\geq |Q||Q'|$, we know that there are multiple choices of $\lambda,\mu$ satisfying the stated construction, and therefore there are at least \emph{two} possible paths $P_{s_0}$ and $\tilde P_{s_0}$ with the desired length, from the initial state to itself, using distinct numbers of cycles \footnote{If the cycles are both a self-loop, then this becomes degenerate, but in the following step the final column and row of the matrix $M$ can be omitted, and the remainder of the argument follows in essentially the same way.} $(\lambda_{s_0}, \mu_{s_0})$ and $(\tilde \lambda_{s_0}, \tilde \mu_{s_0})$.

This construction yields a collection of paths with full rank occupation densities. To verify this explicitly, extract the rows corresponding to the paths $\{R_s\}_{s\in \ST}$ and $\tilde R_{s_0}$, we use \eqref{eq:concatenationoccupation} to see that
\begin{equation}\label{eq:blocklinearsystem}\left[\begin{array}{c}
\mathbb{O}_{T-1}(\{a_t^{P_{s_0}}\}_{t\ge 0})\\
\mathbb{O}_{T-1}(\{a_t^{P_{s_1}}\}_{t\ge 0})\\
\cdots\\
\mathbb{O}_{T-1}(\{a_t^{P_{s_N}}\}_{t\ge 0})\\
\mathbb{O}_{T-1}(\{a_t^{\tilde P_{s_0}}\}_{t\ge 0})
\end{array}\right] = M
\left[\begin{array}{c}
\mathbb{O}_{|r_{s_0}|}(\{a_t^{r_{s_0}}\}_{t\ge 0})\\
\mathbb{O}_{|r_{s_1}|}(\{a_t^{r_{s_1}}\}_{t\ge 0})\\
\cdots\\
\mathbb{O}_{|r_{s_N}|}(\{a_t^{r_{s_N}}\}_{t\ge 0})\\
\mathbb{O}_{|Q|-1}(\{a_t^{Q}\}_{t\ge 0})\\
\mathbb{O}_{|Q'|-1}(\{a_t^{Q'}\}_{t\ge 0})
\end{array}\right] 
\end{equation}
where
\begin{align*}\Gamma(\lambda,Q)&:=\begin{cases} (1-\gamma^{\lambda|Q|})/(1-\gamma^{|Q|}),&\gamma\neq1,\\
\lambda,&\gamma =1,\end{cases}\\
M&=\left[\begin{array}{cccccc}
\gamma^{T'} & 0 & \cdots & 0 & \Gamma(\lambda_{s_0},Q) & \gamma^{\lambda_{s_0}|Q|}\Gamma(\mu_{s_0},{Q'})\\
0&\gamma^{T'-|r_{s_1}|} &  \cdots & 0 & \Gamma(\lambda_{s_1},Q) & \gamma^{\lambda_{s_1}|Q|}\Gamma(\mu_{s_1}{Q'})\\
& &  \ddots &  &  & \\
0& 0 &  \cdots & \gamma^{T'-{|r_{s_N}|}} &\Gamma(\lambda_{s_N},Q) & \gamma^{\lambda_{s_N}|Q|}\Gamma(\mu_{s_N}{Q'})\\
\gamma^{T'} & 0 & \cdots & 0 & \Gamma(\tilde\lambda_{s_0},Q) & \gamma^{\tilde\lambda_{s_0}|Q|}\Gamma(\tilde\mu_{s_0},{Q'})\\
\end{array}\right].
\end{align*}
After subtracting the first from the last row of $M$, as $\lambda_{s_0}\neq \tilde\lambda_{s_0}$, we see that $M$ has a simple structure, in particular it is a full-rank matrix with $N+1$ rows and $N+2$ columns. As the final matrix on the right hand side of \eqref{eq:blocklinearsystem} is of rank $N$, this implies that the left hand side of \eqref{eq:blocklinearsystem} is also of rank $N$ (by Sylvester's rank inequality). As the left hand side of \eqref{eq:blocklinearsystem} is a selection of rows from the matrix considered in Definition \ref{def:fullactionrank}, we conclude that our MDP must be of full action rank.

Our collection of paths also shows that our system has full access at horizon $T=T'+1$, and therefore the identification result follows from Theorem \ref{thm:finitehorizonidentification}. By varying $T'$, we see this result holds for any choice of $T\geq |Q||Q'|+\max_s|r_s|$, as desired.
\end{proof}

\begin{ex}
Consider the problem with three states $\ST=\{A,B,C\}$, with possible transitions $A\to \{B,C\}$, $B\to A$ and $C\to B$. Starting in state $A$, the shortest paths are then given by $\{A\}, \{A\to B\}, \{A\to C\}$, and we have cycles $\{A\to B\to A\}$ and $\{A \to C\to B \to A\}$. Writing out the occupation densities of each of these paths (ignoring the terminal state of the two cycles), with $\gamma = 1$, we get the system
\[\begin{array}{c}
\mbox{}
\\
\begin{array}{rl}
\begin{array}{r}
\text{Shortest paths}\left\{\begin{array}{c} \mbox{}\\ \mbox{}\\ \mbox{}\end{array}\right.\\
\text{Cycles (excluding final state)}\left\{\begin{array}{c}\mbox{}\\ \mbox{}\end{array} \right.
\end{array}&
\left[\begin{array}{c}
A\\
A\to B\\
A\to C\\
\hline
A \to B \\
A \to C \to B
\end{array}\right]\end{array}
\end{array}
\Rightarrow \begin{array}{c}
\begin{array}{ccc}
(A&B&C)\end{array}\\
\left[\begin{array}{ccc}
1&0&0\\
1&1&0\\
1&0&1\\
\hline
1&1&0\\
1&1&1\end{array}
\right]\end{array}.\]
This corresponds to the final term on the right hand side of \eqref{eq:blocklinearsystem}. Clearly, the section above the horizontal line (corresponding to the shortest paths) is lower-triangular, and hence of full rank. We prefix our paths by appropriate numbers of cycles, in order to make them the same length. This implies that, with a horizon $T=7=2\times 3 +1$, we consider the paths
{\small
\begin{align*}\left[\begin{array}{c}
A \to C \to B \to A \to C \to B \to A\\
A \to B \to A \to C \to B \to A \to B\\
A \to B \to A \to C \to B \to A \to C\\
A \to B \to A \to B \to A \to B \to A\\
 \cdots
\end{array}\right] \Rightarrow \left[\begin{array}{ccc}
3&2&2\\
3&3&1\\
3&2&2\\
4&3&0\\
& \cdots \end{array}
\right] &= \bigg[e_A^\top\bigg(\sum_{t=0}^{T-1}\gamma^t \prod_{t'=0}^{t-1}\mathbb{T}(a_{t'})\bigg)\bigg]_{\{a_t\} \subset \AC}
\end{align*}}
The matrix of occupation densities shown here is the left hand side of \eqref{eq:blocklinearsystem} and is easily seen to be full rank; the matrix $M$ from \eqref{eq:blocklinearsystem} is given by 
\[M = \left[\begin{array}{ccccc}
1&0&0&0&2\\
0&1&0&1&1\\
0&0&1&1&1\\
1&0&0&3&0
\end{array}\right].\]
\end{ex}

We can extend this result to a stochastic setting, assuming that our action space is sufficiently rich.
\begin{coro}\label{coro:fullactionstochastic}
Suppose $\gamma\neq0$, and our MDP is stochastic and satisfies one of the sets of assumptions (\ref{caseaperiodic}, \ref{casecycles} or  \ref{caseselfloop}) of Corollary \ref{cor:DeterministicFiniteIdentifiable} and that from every state, we have at least as many actions (with linearly independent resulting transition probabilities) as we have possible future states, that is,
\[\mathrm{rank}\big\{\transprob(\cdot|s,a); a\in \AC\big\} = \#\big\{s':\transprob(s'|s,a)>0 \text{ for some }a\in \AC\big\}.\]
Then for any initial state $s_0$, there exists a horizon $T$ such that the time-homogeneous IRL problem is well posed (as in Theorem \ref{thm:finitehorizonidentification}). The sufficient bounds on $T$ from Corollary \ref{cor:DeterministicFiniteIdentifiable} also apply.
\end{coro}
\begin{proof} 
For a given state $s$, consider the space spanned by the basis vector corresponding to the possible future states. Given we have as many actions as possible future states, and the rank-nullity theorem, we know that this space must be the same as the space spanned by the vectors $\{\transprob(\cdot|s,a); a\in \AC\}$. In particular, there exists a set of weights $c_a$ over actions (which do not need to sum to one or be nonnegative) such that $\sum_{a\in\bar\AC} c_a\mathbb{T}(a)$ is the basis vector corresponding to any possible transition. In other words, there is no difference between the linear span generated by these stochastic transitions and deterministic transitions. As actions at every time can be varied independently, and the requirement that a MDP has full action rank depends only on the space spanned by transition matrices, the problem reduces to the setting of Corollary \ref{cor:DeterministicFiniteIdentifiable}.
\end{proof}

\section{Action-independent rewards}
Earlier works such as \cite{amin2016towards}, \citet{amin2017repeated}, \cite{dvijotham2010} and \cite{fu2017learning} consider the case of action-independent rewards, that is, where $f$ is not a function of $a$. In general, it is not immediately clear whether, for a given observed policy, the IRL problem will admit an action-independent solution. In this section, we obtain a necessary and sufficient condition under which an action-independent time-homogeneous reward function could be a solution to a given entropy-regularized, infinite-time-horizon\footnote{The analogous results for finite-horizon problems with time-inhomogeneous rewards (and general discount factor) can be obtained through the same method.} IRL problem with discounting. We shall also obtain a rigorous condition under which a unique reward function can be identified. 

Consider an entropy-regularized MDP environment $(\ST,\AC,\transprob,\gamma,\lambda)$, as given in Section \ref{sec:MDP}. Without loss of generality, assume that $|\ST|,|\AC|\geq2$ and $\ST=\{s_1,\dots,s_{|\ST|}\}$. Let $\bar \pi:\ST\to\mathcal{P}(\AC)$ be the observed optimal policy such that $\bar\pi(a|s)>0$ for any $(s,a)\in\ST\times\AC$. 

As before, for $a\in\AC$ we write $\bar\pi(a)\in\mathbb{R}^{|\ST|}$ for the probability vector $\begin{pmatrix}\bar\pi(a|s_1)&\dots&\bar\pi(a|s_{|\ST|})\end{pmatrix}^T$, and  $\mathbb{T}(a)\in\mathbb{R}^{|\ST|\times|\ST|}$ for the transition matrix with $[\mathbb{T}(a)]_{ij}=\transprob(s_j|s_i,a)$. Fix a particular action $a_0\in\AC$, and write \[\Delta\log\bar\pi(a)=\log\bar\pi(a)-\log\bar\pi(a_0)\quad\text{ and }\quad\Delta\mathbb T(a)=\mathbb T(a)-\mathbb T(a_0),\] where $\log\bar\pi(a)$ denotes the element-wise application of logarithm over the vector $\bar\pi(a)$, for any $a\in\AC$.
\begin{thm}\label{thm:actionindeptexist}
 The above IRL problem admits a solution with action-independent reward $f:\ST\to\mathbb{R}$ if and only if the system of equations
 \begin{equation}
     \label{eq: sys-eqns}
     \lambda\Delta\log\bar\pi(a)=\gamma\Delta\mathbb{T}(a)\upsilon,\quad \forall a\in\AC,
 \end{equation}
 admits a solution $\upsilon\in\mathbb{R}^{|\ST|}$. (Note that this is a system of $|\AC|\times |\ST|$ equations in $|\ST|$ unknowns, so this is a non-trivial assumption.)
\end{thm}
\begin{proof} 
\begin{description}
    \item[Necessity.] Suppose the IRL problem admits an action-independent solution $f:\ST\to\mathbb{R}$. Then for any $(s,a)\in\ST\times\AC$,
    \[f(s)=\lambda\log\bar\pi(a|s)-\gamma\sum_{s'\in\ST}\transprob(s'|s,a)v(s')+v(s),\]
    where $v$ is the corresponding value function. Notice that for any $a\in\AC$, for all $s\in \ST$,
    \[\begin{aligned}f(s)&=\lambda\log\bar\pi(a|s)-\gamma\sum_{s'\in\ST}\transprob(s'|s,a)v(s')+v(s)\\
    &=\lambda\log\bar\pi(a_0|s)-\gamma\sum_{s'\in\ST}\transprob(s'|s,a_0)v(s')+v(s).\end{aligned}\]
    Therefore, taking $\upsilon$ to be the vector with components $v(s)$, we have a solution to the system of equations \eqref{eq: sys-eqns}.
    \item[Sufficiency.] Let $\upsilon$ be a solution to the system of equations \eqref{eq: sys-eqns}. By abuse of notation, we may write $\upsilon(s)$ for the components of $\upsilon$. Then for any $(s,a)\in\ST\times\AC$,
    \[\begin{aligned}
    &\lambda\log\bar\pi(a|s)-\gamma\sum_{s'\in\ST}\transprob(s'|s,a)\upsilon(s')=\lambda\log\bar\pi(a_0|s)-\gamma\sum_{s'\in\ST}\transprob(s'|s,a_0)\upsilon(s').
    \end{aligned}\]
    Therefore, the quantity $\hat f(s):=\lambda\log\bar\pi(a_0|s)-\gamma\sum_{s'\in\ST}\transprob(s'|s,a_0)\upsilon(s')+\upsilon(s)$ is independent of $a$. From Theorem 1, we conclude that $\hat f$ is a solution to the IRL problem.
\end{description}
\end{proof}

\begin{coro} \label{coro:actionindeptunique}
Suppose $\gamma\in[0,1)$. Assuming a solution to \eqref{eq: sys-eqns} exists, the IRL problem is identifiable (i.e. the true action-independent reward function can be inferred up to a constant shift) if and only if, writing $\mathcal{K}(a)$ for the kernel of $\Delta\mathbb{T}(a)$, we know that
\[\{c\mathbf{1}:c\in\mathbb{R}\}=\bigcap_{a\in\AC\setminus\{a_0\}}\mathcal{K}(a)=\bigcap_{a\in\AC}\mathcal{K}(a),\]
where $\mathbf{1}$ denotes the all-one vector in $\mathbb{R}^{|\ST|}$.
(Note that $\{c\mathbf{1}:c\in\mathbb{R}\}\subset \mathcal{K}(a)$ for any $a\in\AC\setminus\{a_0\}$ and $\mathbb{T}(a_0)=0$ implies $\mathcal{K}(a) =\mathbb{R}^{|\ST|}$.)
\end{coro}
\begin{proof}
Let $\upsilon_0$ be a solution to \eqref{eq: sys-eqns}, which is assumed to exist. By the Fredholm alternative (as in Theorem \ref{thm:IRL-MDP2}) the solution set $\mathbb{Y}_\ST$ for \eqref{eq: sys-eqns} is given by
\[
\mathbb{Y}_\ST=\bigg\{\upsilon_0+\kappa:\kappa\in{\rm span}\bigg(\bigcap_{a\in\AC}\mathcal{K}(a)\bigg)\bigg\}.
\]
From Theorem \ref{thm:actionindeptexist}, the set of action-independent solutions for the IRL is given by
\[\mathbb{F}_{\ST}=\bigg\{f:
f(s)=\lambda\log\bar\pi(a_0|s)-\gamma\sum_{s'\in\ST}\transprob(s'|s,a_0)\upsilon(s')+\upsilon(s);\quad \text{for }  \upsilon\in\mathbb{Y}_\ST, s\in\ST\bigg\}.
\]

We then observe that the stated condition is sufficient -- if constant vectors are the only valid choices for $\kappa$, then $\upsilon$ and hence $f\in \mathbb{F}_\ST$ will only vary by constants.

To show necessity, denote by $f_0$ the solution corresponding to $\upsilon_0$. Suppose there exists a vector 
\[\hat \upsilon\in\bigg(\bigcap_{a\in\AC}\mathcal{K}(a)\bigg)\setminus\{c\mathbf{1}:c\in\mathbb{R}\}.\] Define 
\[\Delta(s)=\hat \upsilon(s)-\gamma\sum_{s'\in\ST}\transprob(s'|s,a_0)\hat \upsilon(s'),\quad \forall s\in\ST.\]
It follows that $f_0+\Delta\in\mathbb{F}_{\ST}$; if $\Delta$ is not a constant, we see that the reward is not uniquely identifiable.

To show $\Delta$ is not a constant, let 
\[
\overline{\upsilon}=\max_{s\in\ST}\hat \upsilon(s),\quad \overline{s}\in\argmax_{s\in\ST}\hat \upsilon(s),\quad \underline{\upsilon}=\min_{s\in\ST}\hat \upsilon(s),\quad \underline{s}=\argmin_{s\in\ST}\hat \upsilon(s),\quad \tilde{\upsilon}=\frac{\sum_{s\in\ST}\hat \upsilon(s)}{|\ST|}.
\]
Then $\underline{\upsilon}=\hat{\upsilon}(\underline{s})<\tilde{\upsilon}<\overline{\upsilon}=\hat{\upsilon}(\overline{s})$. We have
\[
\begin{aligned}
&\Delta(\overline{s})-(1-\gamma)\tilde{\upsilon}=\overline{\upsilon}-\tilde{\upsilon}-\gamma\sum_{s\in\ST}\transprob(s|\overline{s},a_0)[\hat \upsilon(s)-\tilde{\upsilon}]\geq (1-\gamma)(\overline{\upsilon}-\tilde{\upsilon})>0;\\
&\Delta(\underline{s})-(1-\gamma)\tilde{\upsilon}=\underline{\upsilon}-\tilde{\upsilon}-\gamma\sum_{s\in\ST}\transprob(s|\underline{s},a_0)[\hat \upsilon(s)-\tilde{\upsilon}]\leq (1-\gamma)(\overline{\upsilon}-\tilde{\upsilon})<0.
\end{aligned}
\]
Therefore, $\Delta$ is not a constant. It follows that our condition is necessary in order to have an identifiable action-independent reward
\end{proof}

Under the entropy regularized framework, the long-run total reward depends on actions through the entropy penalty term. Therefore, it cannot be reduced to the scenario in \cite{amin2016towards}, where any linear perturbation of the reward function will not affect optimal behavior under any given environment.
\begin{rmk}
Theorem \ref{thm:actionindeptexist} and Corollary \ref{coro:actionindeptunique} suggest various extensions, in the case when \eqref{eq: sys-eqns} does not admit a solution, but the assumed property on the kernels in Corollary \ref{coro:actionindeptunique} holds. For example, one could consider the least-squares solution to the system \eqref{eq: sys-eqns} (which is defined up to a constant). This gives a choice of value function which, in some sense, minimizes the action-dependence of the resulting cost function (obtained through Theorem \ref{thm:IRL-MDP1}).
\end{rmk}

\cite{fu2017learning} give a result similar to Corollary \ref{coro:actionindeptunique}. Unfortunately, the role of the choice of actions in their conditions is not precisely stated, and on some interpretations is insufficient for the result to hold -- as we have seen, the condition of Corollary \ref{coro:actionindeptunique} is both necessary and sufficient for identifiability. We give a variation of their assumptions in what follows.

\begin{defn}[Reward-decomposability]\label{defn:rewarddecomp}
We say states $s_1, s_1'$ are `1-step linked', if there exist actions $a, a'\in\AC$ and a state $s_0\in\ST$ such that $\transprob(s_1|s_0,a)>0$ and $\transprob(s'_1|s_0,a')>0$.
We extend this definition through transitivity, forming a set of `linked' states $\mathcal{S}_1$. We say say the MDP is reward-decomposable if all its states are linked.
\end{defn}
Note that there is no loss of generality if a specific $a'$ is selected in this definition (instead of being allowed to vary).

\begin{rmk} An equivalent definition would be that our MDP is reward-decomposable if  $\ST_1=\ST$ is the only choice of nonempty set $\ST_1\subset \ST$ such that: there exists a set $\ST_0\subset\ST$ with
\begin{enumerate}[label = (\roman*)]
    \item every transition (with any action) to $\ST_1$ is from $\ST_0$, and
    \item every transition from $\ST_0$ is to $\ST_1$. 
\end{enumerate}(In other words, $X_{t}\in \ST_0$ if and only if $X_{t+1}\in \ST_1$.) We note that \cite{fu2017learning} simply call this property `decomposable', but this seems an unfortunate choice of terminology given this alternative characterization.
\end{rmk}

 The following final corollary gives a simple set of conditions under which identification is possible, clarifying (and extending to the stochastic case) the result of \cite[Theorem C.1]{fu2017learning}.

\begin{coro}\label{cor:actionindependentdecomposable}
Suppose our MDP either has deterministic transitions $\mathcal{T}(s'|s,a)\in\{0,1\}$ or we have at least as many actions (with linearly independent resulting transition probabilities) as we have possible future states, that is,
\[\mathrm{rank}\big\{\transprob(\cdot|s,a); a\in \AC\big\} = \#\big\{s':\transprob(s'|s,a)>0 \text{ for some }a\in \AC\big\}.\]
Then the (action-independent) IRL problem is identifiable (i.e. the true action-independent reward function can be inferred up to a constant shift) if and only if the MDP is reward-decomposable.
\end{coro}
\begin{proof}
We will verify the condition of Corollary \ref{coro:actionindeptunique}.

In the stochastic transition case, the proof of Corollary \ref{coro:fullactionstochastic} shows that, under the stated assumption on the rank of the transitions, we can perform row operations on our transition matrix (corresponding to linear combinations of actions) to obtain a  deterministic transition matrix. In particular, the dimension of $\cap_{a\in\AC}\mathcal{K}(a)$ is the same under the assumption on the rank of the transitions as under the assumption that transitions are deterministic. We can therefore focus our attention on the deterministic transition case.

If transitions are deterministic, the matrix $\mathbb{T}(a)$ has rows given by the basis vectors indicating the future states; so the matrix $\Delta \mathbb{T}(a)$ has rows which are the difference of two basis vectors corresponding to one-step linked states. Therefore, a vector $\upsilon\in\mathcal{K}(a)$ must have entries $\upsilon_i=\upsilon_j$ whenever $e_i$ and $e_j$ correspond to these one-step-linked states.

By considering all possible choices of $a$, we see that a vector $\upsilon \in \cap_{a\in\AC}\mathcal{K}(a)$ must have entries $\upsilon_i = \upsilon_j$ whenever $e_i$ and $e_j$ correspond to any one-step linked states (and this is a sufficient condition to ensure $\upsilon\in\cap_{a\in\AC}\mathcal{K}(a)$). However, if our MDP is reward-decomposable, there is no proper subset $\ST_1$ of $\ST$ which is closed under taking one-step linked states. Therefore, if our MDP is reward-decomposable, the only vectors in the kernel of $\Delta\mathbb{T}(a)$ for every $a$ are the constant vectors, as desired.

Conversely, if our MDP is not reward-decomposable, then there exists a set $ \ST_1\neq \ST$ satisfying the conditions above, and hence a nonconstant vector $\upsilon\in\cap_{a\in\AC}\mathcal{K}(a)$. The result of Corollary \ref{coro:actionindeptunique} then shows the reward is not identifiable.
\end{proof}

It is clear that reward-decomposability is not, by itself, sufficient to guarantee identifiability of rewards -- simply consider the trivial MDP with action space containing only one element (so no information can be gained by watching optimal policies) but all transitions are possible (so the MDP is reward-decomposable).

The necessity of reward-decomposability, in general, can easily be seen as follows: Suppose there are sets $\ST_0, \ST_1 \subset \ST$ such that every transition from a state in $\ST_0$ (under every action) is to a state in $\ST_1$, and every transition to a state in $\ST_1$ is from $\ST_0$. Then, if we add $c\in \mathbb{R}$ to the reward in $\ST_0\setminus\ST_1$, subtract $c/\gamma$ from the reward in $\ST_1\setminus\ST_0$, and add $(1-1/\gamma)c$ to the reward in state $\ST_0\cap \ST_1$, we will have no impact on the overall value or optimal strategies. A reward-decomposability assumption ensures $\ST_0 = \ST$ (which implies  $\ST_0 = \ST$ as every transition into $\ST_1$ must be from a state in $\ST_0$), so this is simply a constant shift; otherwise, we see our IRL problem is not identifiable.

\section{A linear-quadratic-Gaussian problem}
We now present the corresponding results for a class of one-dimensional linear-quadratic problems with Gaussian noise, ultimately inspired by \citet{Kalman64}. This simplified framework allows us to explicitly observe the degeneracy of inverse reinforcement learning, even if we add restrictions on the choice of value functions.

\paragraph{Optimal LQG control}
Suppose our agent seeks to control, using a real-valued process $A_t$ a discrete time process with dynamics
\[S_{t+1} = (\bar\mu+ \mu_s S_{t} + \mu_a A_t) + (\bar \sigma + \sigma_s S_t + \sigma_a A_t) Z_{t+1}\]
for constants $\bar\mu,  \mu_s, \mu_a, \bar\sigma, \sigma_s, \sigma_a$. The innovations process $Z$ is a Gaussian white noise with unit variance. Our agent uses a randomized strategy $\pi(a|s)$ to maximize the expectation of the entropy-regularized infinite-horizon discounted linear-quadratic reward:
\[\mathbb{E}\bigg[\sum_{t=1}^\infty \gamma^t \bigg(\int_\rn f(S_t,a) \pi(a|s)da+ \lambda\mathcal{H}(\pi(\cdot|S_t)\bigg)\bigg]\]
where $f(s,a) = \alpha_{20}s^2+\alpha_{11}sa+\alpha_{02}a^2+\alpha_{10}s+\alpha_{01}a+\alpha_{00}$
and $\mathcal{H}(\pi) = -\int_\rn \pi(a)\log(\pi(a)da$ is the Shannon entropy of $\pi$. We assume the coefficients of $f$ are such that the problem is well posed (i.e. it is not possible to obtain an infinite expected reward).

Just as in the discrete state and action space setting, we can write down the state-action value function 
\begin{equation}
    \label{eq:LQ-Qfn}
Q^\pi_\lambda(s, a) = f(s,a) + \gamma \int_\rn V^\pi_\lambda(s') \frac{1}{\sqrt{2\pi(\bar\sigma+\sigma_s s+\sigma_a a)^2}}\exp\Big(- \frac{(s'- \mu_s s - \mu_a a)^2}{2(\bar\sigma+\sigma_s s+\sigma_a a)^2}\Big) ds'.
\end{equation}
Using this, the optimal policy and value function are given by 
\begin{align}
\label{eq:LQ-ctrl}
	\pi^{\ast}_{\lambda}(a|s)&=\exp\Big(\big(Q^{\pi^{\ast}_\lambda}_{\lambda}(s,a)-V^{\pi^{\ast}_\lambda}_{\lambda}(s)\big)\big/\lambda\Big),\\
	V^{*}_{\lambda}(s)&=V^{\pi^{\ast}_\lambda}_{\lambda}(s)=\lambda\log\int_\rn \exp\Big(Q^{\pi^{\ast}_\lambda}_{\lambda}(s,a)\big/\lambda\Big)da\,.
\end{align}
What is particularly convenient about this setting is that $Q$ is a quadratic in $(s,a)$,  $V_\lambda$ is a quadratic in $s$, and $\pi^*_\lambda(\cdot|s)$ is a Gaussian density. In particular, the optimal policy is of the form 
\begin{equation}
    \label{eq: lqr-ctrl}
    \begin{split}
    \pi^*_\lambda(a|s)&=\frac{1}{\sqrt{2\pi \lambda k_3}}\exp\left\{-\frac{(a-k_1s - k_2)^2}{2\lambda k_3}\right\}
    \\
    &=\exp\left\{-\frac{1}{\lambda}\left[\frac{a^2-2(k_1s+k_2)a}{2k_3}+\frac{(k_1s+k_2)^2}{2k_3}+\frac{\lambda}{2}\log\left(2\pi k_3\lambda\right)\right]\right\}\end{split}
\end{equation}
for some constants $k_1, k_2\in \rn, k_3>0$, which can be determined\footnote{The explicit formulae for  $k_1, k_2$ and $k_3$, and the coefficients of the value function, can be obtained by equating the coefficients of $f$ with the values obtained in  \eqref{eq:LQ-reward-form}. Under the assumption that the optimal control problem is well posed, this has a solution with $k_3>0$.} in terms of the known parameters $\mu_a, \mu_s, \sigma, \lambda$ and the parameters of the reward function $\{\alpha_{ij}\}_{i+j\leq 2}$.

\begin{thm}\label{thm:lqr-sol-char}
Consider an agent with a policy of the form \eqref{eq: lqr-ctrl}. Suppose we also know  that the value function $V$ is a quadratic (or, equivalently, that the reward function is a quadratic in $(s,a)$). The space of rewards consistent with this policy is given by:
\begin{equation}
    \label{eq: sioc-cand-f}
    \begin{aligned}
    \mathbb{F}=&\bigg\{f(s,a)=a_{20}s^2+a_{11}sa+a_{02}a^2+a_{10}s+a_{01}a+a_{00}\biggl|\\
    &(a_{20},\,a_{11},\,a_{02})=\Big(\frac{-k_1^2}{2k_3},\,\frac{k_1}{k_3},\,\frac{-1}{2k_3}\Big)-\beta_2\Big(\gamma( \mu_s^2+\sigma_s^2) -1, \, 2\gamma(\mu_s\mu_a+\sigma_s\sigma_a),\, \gamma(\mu_a^2+\sigma_a^2)\Big),\\
    &(a_{10},\,a_{01})=\Big(\frac{-k_1k_2}{k_3},\, \frac{k_2}{k_3}\Big)-\beta_2\Big(2\gamma(\bar\mu\mu_s+\bar\sigma\sigma_s), 2\gamma(\bar \mu\mu_a+\bar\sigma\sigma_a)\Big)-\beta_1\Big(\gamma\mu_s+1,\,\gamma\mu_a\Big),\\
    &a_{00},\beta_2,\beta_1\in\rn\bigg\}.
    \end{aligned}
\end{equation}
\end{thm}
\begin{proof}
We consider an arbitrary quadratic
\[v(s) = \beta_2 s^2 + \beta_1 s + \beta_0\]
as a candidate value function. If we have begun from the assumption that the reward function $f$ is quadratic, we know that the corresponding value function is quadratic, so this is not a restrictive assumption.

We then compute the state-action value function using \eqref{eq:LQ-Qfn}, to give
\[Q_\lambda(s,a) = f(s,a) + \gamma \Big(\beta_2\big((\bar \mu+ \mu_ss + \mu_aa)^2+(\bar\sigma+\sigma_s s+\sigma_a a)^2\big)+ \beta_1(\bar \mu+\mu_ss + \mu_aa) + \beta_0 \Big).\]
Combining with \eqref{eq: lqr-ctrl} and \eqref{eq:LQ-ctrl}, we see that a reward function $f$ is consistent with the observed policy if
\begin{align*}
\lambda\log \pi(a|s) & = -\bigg[\frac{a^2-2(k_1s+k_2)a}{2k_3}+\frac{(k_1s+k_2)^2}{2k_3}+\frac{\lambda}{2}\log\left(2\pi k_3\lambda\right)\bigg] \\
&= f(s,a) + \gamma \Big(\beta_2\big((\bar \mu+ \mu_ss + \mu_aa)^2+(\bar\sigma+\sigma_s s+\sigma_a a)^2\big)+ \beta_1(\bar \mu+ \mu_ss + \mu_aa) + \beta_0 \Big)\\
&\quad -\Big(\beta_2 s^2 + \beta_1 s + \beta_0\Big).
\end{align*}
Rearranging, we conclude that $f$ is given by 
\begin{equation}\label{eq:LQ-reward-form}\begin{split}
f(a,s) &= 
\Big[-\frac{k_1^2}{2k_3} -  \beta_2(\gamma(\mu_s^2+\sigma_s^2) - 1)\Big]s^2
+ \Big[\frac{k_1}{k_3}-2\beta_2\gamma(\mu_s\mu_a+\sigma_s\sigma_a)\Big] as\\
&\quad + \Big[-\frac{1}{2k_3} - \beta_2\gamma(\mu_a^2+\sigma_a^2)\Big]a^2
+ \Big[-\frac{k_1k_2}{k_3}-2\beta_2\gamma(\bar\mu\mu_s+\bar\sigma\sigma_s)-\beta_1(\gamma\mu_s+1) \Big]s\\
&\quad + \Big[\frac{k_2}{k_3} - 2\beta_2\gamma(\bar\mu\mu_a + \bar \sigma \sigma_a) -\beta_1(\gamma\mu_a)\Big]a\\&\quad  + \Big[-\frac{\lambda}{2}\log(2\pi k_3 \lambda)-\beta_2\gamma( \bar\mu^2+\bar\sigma^2) -\beta_1\bar\mu+ \beta_0(1-\gamma)\Big].
\end{split}\end{equation}
As $(\beta_2, \beta_1, \beta_0)$ are arbitrary, we have the desired statement.
\end{proof}

As in Theorem \ref{thm:IRL-MDP1}, we see that the inverse reinforcement learning problem only defines the rewards up to the choice of value function, which is arbitrary; the restriction to quadratic rewards or values simply reduces our problem to the smaller range of rewards determined by the three coefficients in the quadratic $V$.

The following theorem gives the linear-quadratic version of Theorem \ref{thm:IRL-MDP2}. As our agents' actions have a linear effect on the state variable, this leads to a particularly simple set of conditions for identifiability of the reward, given observation of two agents' policies.

\begin{thm}\label{thm:twoagents}
Suppose we now have two agents, who are both following their respective optimal controls of the form \eqref{eq: lqr-ctrl}, for the same reward function, but disagree on some combination of the dynamics and discount rate. We write 
\[x_1=\Big(\gamma \mu_s+1,\,\gamma \mu_a\Big) \qquad \text{and}\qquad x_2=\Big(\gamma (\mu_s^2+\sigma_s^2)-1,\, 2\gamma(\mu_s\mu_a+\sigma_s\sigma_a),\, \gamma (\mu_a^2+\sigma_a^2)\Big),\]
giving us two pairs of vectors $(x_1, x_2)$ (for the first agent) and $(\tilde x_1, \tilde x_2)$ (for the second agent). We assume we know these vectors for each agent.
The quadratic reward function $f$ consistent with both agents' policies, if it exists, is uniquely identified up to the addition of a constant shift, if (and only if) 
\[\frac{x_1}{\|x_1\|}\neq \frac{\tilde x_1}{\|\tilde x_1\|} \qquad \text{and}\qquad \frac{x_2}{\|x_2\|}\neq \frac{\tilde x_2}{\|\tilde x_2\|}.\]
\end{thm}
\begin{proof}
We see from \eqref{eq: sioc-cand-f} that a single agent's actions identify a space of valid rewards $\mathbb{F}$, which is parameterized by the constant shift $a_{00}$ and the two free variables $\beta_1, \beta_2$. From these free variables, \eqref{eq: sioc-cand-f} identifies the values of $\mathbf{a}=(a_{20}, a_{11}, a_{02}, a_{10}, a_{01}).$ The reward function $f$ is uniquely defined, up to a constant shift, if we can identify the value of $\mathbf{a}$, which (by assumption) is the same for both agents.

Considering the role of $\beta_2$, \eqref{eq: sioc-cand-f} defines a line in $\mathbb{R}^3$ of possible values for $(a_{20}, a_{11}, a_{02})$. If the assumption $x_2/\|x_2\| \neq \tilde x_2/\|\tilde x_2\|$ holds, then the lines for our two agents will not be parallel, therefore will either never meet (in which case no consistent reward exists), or will meet at a point, uniquely identifying $(a_{20}, a_{11}, a_{02})$ and the corresponding values of $\beta_2$ for each agent. Conversely, if the assumption does not hold, then the lines will be parallel, so cannot meet in a unique point, in which case there are either zero or infinitely many reward functions consistent with both agents' policies.

Essentially the same argument then applies to the equation for $(a_{10}, a_{01})$. Given that $\beta_2$ has already been identified for each agent, varying $\beta_1$ for each agent defines a pair of lines in $\mathbb{R}^2$, which are not parallel if and only if the stated assumption on $x_1, \tilde x_1$ holds. Therefore, we can uniquely identify $(a_{10}, a_{01})$ if and only if the stated assumption holds. 
\end{proof}

Due to the simplicity of the characterization in Theorem \ref{thm:twoagents}, we can easily see that it is enough to observe two agents using different discount rates. 
\begin{coro}\label{cor:LQdiscount}
Suppose we observe two agents, each using optimal policies of the form \eqref{eq: lqr-ctrl}, for the same dynamics and rewards, but different discount rates. Then the underlying quadratic reward consistent with both agents' policies is identifiable up to a constant.
\end{coro}
\begin{proof}
Simply observe that the value of $\gamma$ introduces a non-scaling change in the vectors $x_1, x_2$ defined in Theorem \ref{thm:twoagents}.
\end{proof}

We can also easily determine the identifiability of action-independent rewards.
\begin{coro}
For an agent with a policy of the form \eqref{eq: lqr-ctrl}, there exists an action-independent reward function corresponding to this policy if and only if 
\[k_1 = - \frac{\mu_s\mu_a +\sigma_s\sigma_a}{\mu_a^2 + \sigma_a^2}\]
and this case, the action-independent reward is unique.
\end{coro}
\begin{proof}
From Theorem \ref{thm:lqr-sol-char}, in order to have an action independent reward we must have $a_{11} = a_{02} = a_{01}=0$. From \eqref{eq: sioc-cand-f}, we know
\begin{align*}
    a_{11}=0 &\quad\Rightarrow\quad \beta_2 = \frac{k_1}{2k_3\gamma(\mu_s\mu_a+\sigma_s\sigma_a)},\\
    a_{02}=0 &\quad\Rightarrow\quad \beta_2 = \frac{-1}{2k_3\gamma(\mu_a^2 + \sigma_a^2)}.
\end{align*}
The statement $k_1 = - (\mu_s\mu_a +\sigma_s\sigma_a)/(\mu_a^2 + \sigma_a^2)$ is easily seen to be equivalent to stating that these equations are consistent.

The value of $\beta_1$ can then always be chosen in a unique way to guarantee $a_{01}=0$, as required.
\end{proof}

\section{Numerical examples of inverse reinforcement learning}
In this section, we present a regularized MDP as in Section \ref{sec:IRLMDP} to illustrate numerically the identifiability issue associated with inverse RL. In particular, we consider a state space $\ST$ with $10$ states and an action space $\AC$ with $5$ actions, with $\lambda = 1$. We compute optimal policies as in Section \ref{sec:MDP} and reconstruct the underlying rewards. As discussed in Section \ref{sec:IRL} the optimal policies and the transition kernel can be inferred from state-action trajectories, so will assumed known. We identify the state and action spaces with the basis vectors in $\rn^{10}$ and $\rn^5$ respectively, so can write $f(a,s) = a^\top R s$ for the reward function, and $\transprob(s'|s,a)= s^\top P_a(s')$ for the transition function. The true reward $R_{{\rm tr}}$ and transition matrices $\{P_a\}_{a\in \AC}$ are randomly generated and fixed; see Figures \ref{fig: true-R} and \ref{fig: true-A}.

\subsection{Non-uniqueness of infinite-sample IRL}
We first look at inverse RL starting from a single optimal policy $\pi_1$ with discount factor $\gamma_1=0.95$. We represent $\pi_1$ as a matrix $\Pi_1$ in $\rn^{5\times 10}$, where each column gives the probabilities of each action when in the corresponding state.

\begin{figure}[!htp]
    \centering
    \includegraphics[width = .5\textwidth]{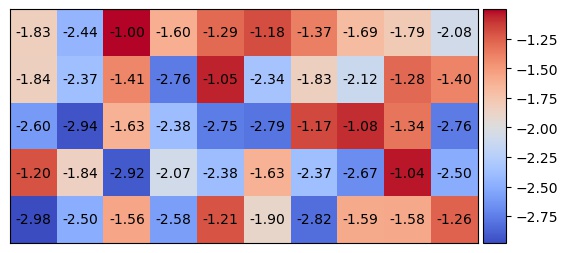}
    \caption{Underlying true reward matrix $R_{{\rm tr}}$}
    \label{fig: true-R}
\end{figure}

\begin{figure}[!htp]
    \centering
    \begin{subfigure}[b]{0.3\textwidth}
        \centering
        \includegraphics[width=\textwidth]{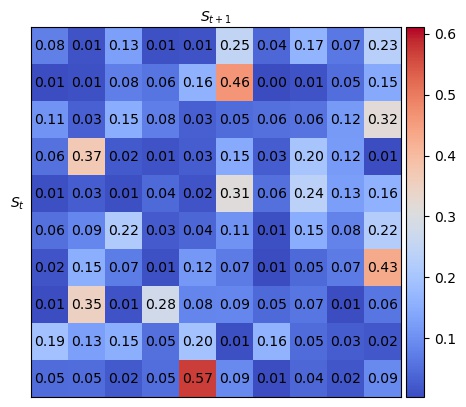}
        \caption{$P_{\mathbf{e}^5_1}$}
        \label{subfig: A-1}
    \end{subfigure}
    \begin{subfigure}[b]{0.3\textwidth}
        \centering
        \includegraphics[width=\textwidth]{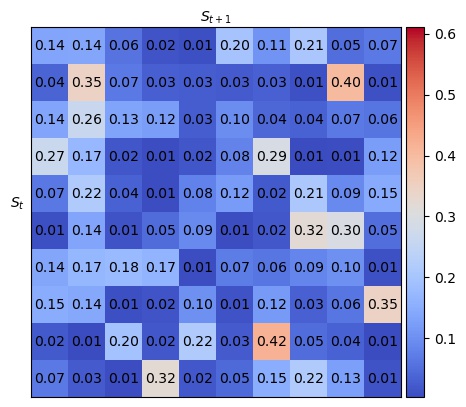}
        \caption{$P_{\mathbf{e}^5_2}$}
        \label{subfig: A-2}
    \end{subfigure}
    \begin{subfigure}[b]{0.3\textwidth}
        \centering
        \includegraphics[width=\textwidth]{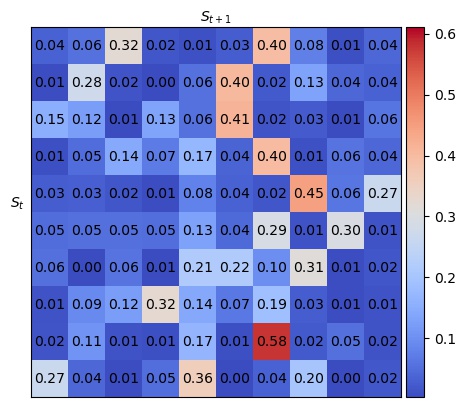}
        \caption{$P_{\mathbf{e}^5_3}$}
        \label{subfig: A-3}
    \end{subfigure}
    \\
    \begin{subfigure}[b]{0.3\textwidth}
        \centering
        \includegraphics[width=\textwidth]{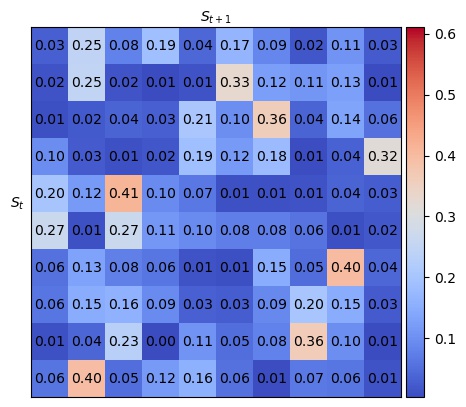}
        \caption{$P_{\mathbf{e}^5_4}$}
        \label{subfig: A-4}
    \end{subfigure}
    \begin{subfigure}[b]{0.3\textwidth}
        \centering
        \includegraphics[width=\textwidth]{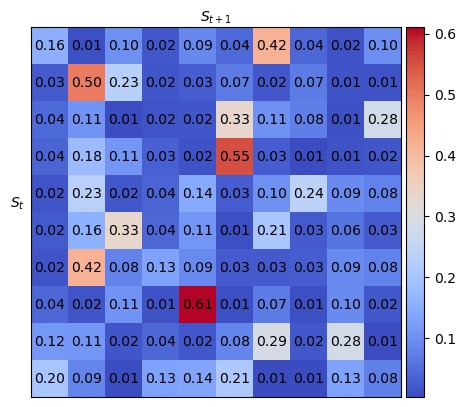}
        \caption{$P_{\mathbf{e}^5_5}$}
        \label{subfig: A-5}
    \end{subfigure}
    \caption{Underlying true transition kernel}
    \label{fig: true-A}
\end{figure}

To solve the inverse RL problem, we numerically find $R,v$ to minimize the loss  
\[L_{{\rm sing}}(R,v)=\sum_{a\in\AC}\sum_{s\in\ST}\left[a^\top \Pi_1s-\exp\left\{a^\top Rs+\gamma_1 s^\top P_a v-v^\top s\right\}\right]^2.\]
An Adam optimizer is adopted with $\alpha=0.002$, $(\beta_1,\beta_2)=(0.5,0.9)$ with overall 2000 minimization steps. The experiments are conducted over 6 different random initializations, sampled from the same distribution as was used to construct the ground truth model. The training loss $L_{{\rm sing}}$ decays rapidly to close to 0, as shown in Figure \ref{fig: sing-trloss}. This indicates that, after the minimization procedure comes to an end, the learnt reward matrix $\hat R$ reveals a corresponding optimal policy $\hat \Pi_1$ close to the true optimal policy $\Pi_1$; see also Figure \ref{fig: sing-char} for a direct comparison.

However, when comparing the learnt reward $\hat R$ and the underlying reward $R_{{\rm tr}}$, as in Figures \ref{fig: sing-C} and \ref{subfig: sing-comp-c}, as well as the comparison between the corresponding value vectors as in Figure \ref{subfig: sing-comp-v}, we can see that the true reward function $R_{{\rm tr}}$ has not been correctly inferred. Here this is not an issue of statistical error, as we assume full information on the optimal policy and the Markov transition kernel.

\begin{figure}[!ht]
\centering
    \begin{subfigure}[b]{0.47\textwidth}
        \centering
        \includegraphics[width=\textwidth]{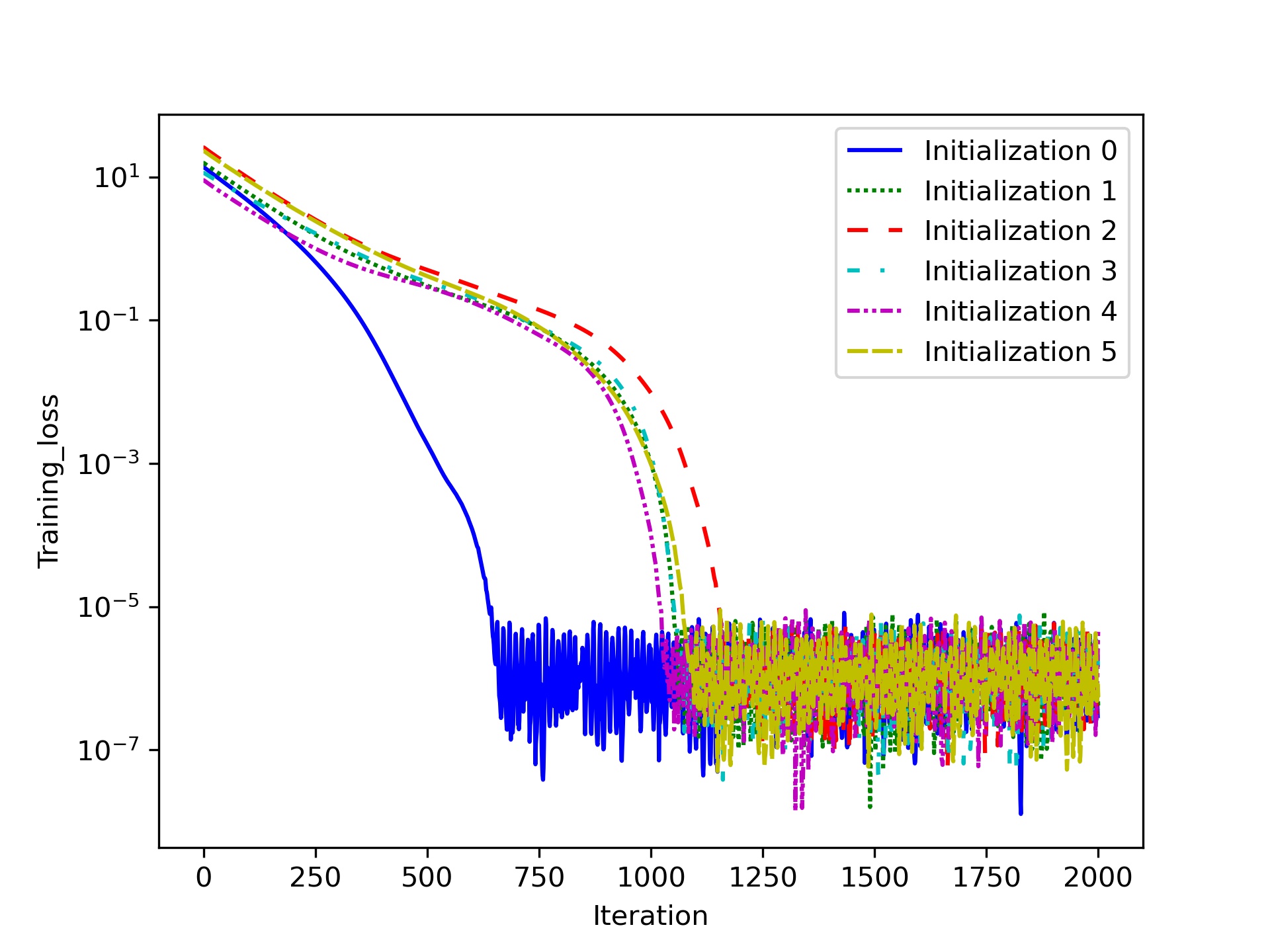}
        \caption{Loss with one optimal policy, under $\gamma_1$}
        \label{fig: sing-trloss}
    \end{subfigure}
    \qquad
    \begin{subfigure}[b]{0.47\textwidth}
        \centering
        \includegraphics[width=\textwidth]{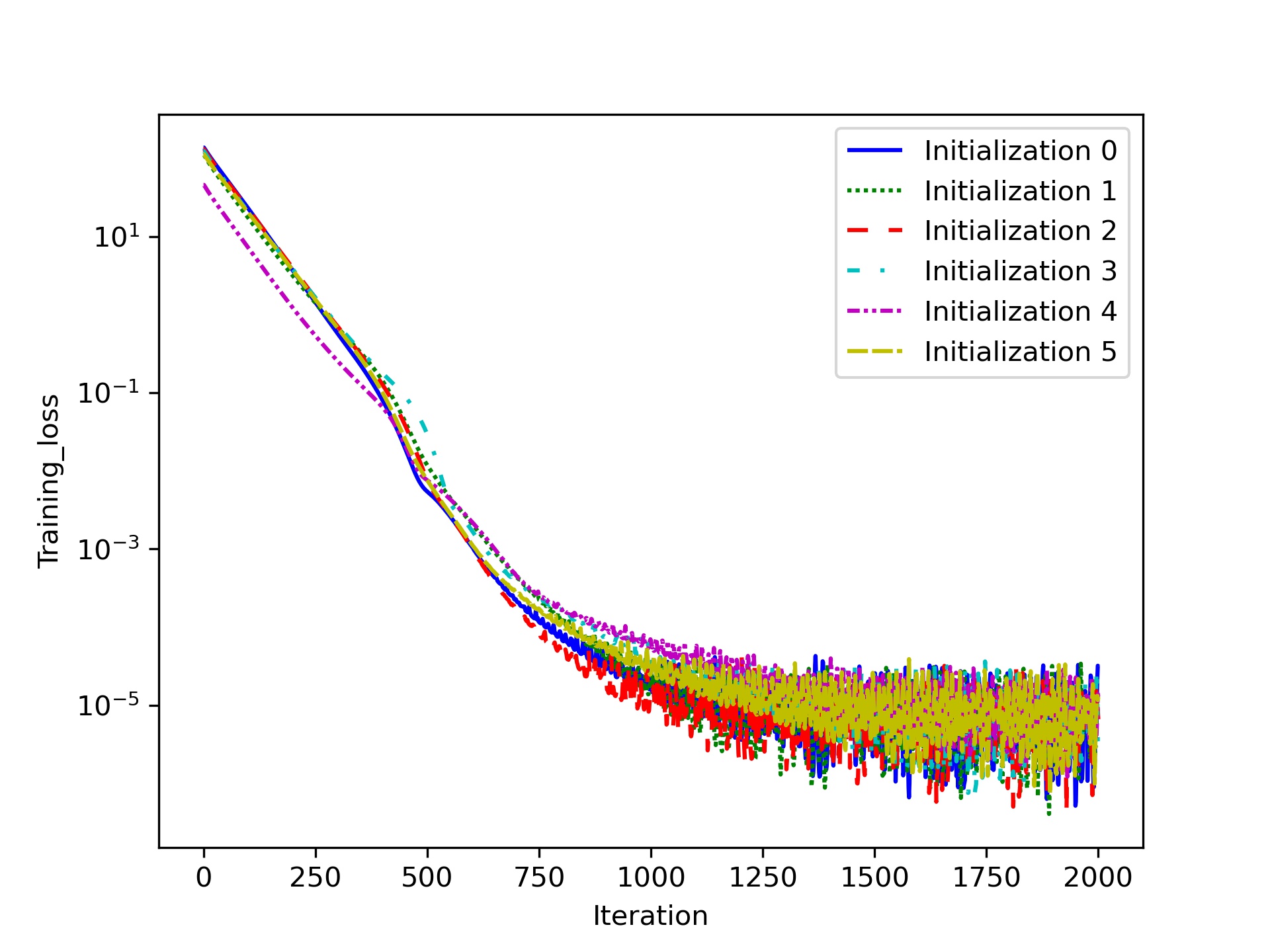}
        \caption{Loss with two optimal policies, under $\gamma_1$ and $\gamma_2$}
        \label{fig: doub-trloss}
    \end{subfigure}
    \caption{Training Losses}
    \label{fig: tr-loss}
\end{figure}

\begin{figure}[!ht]
    \centering
    \begin{subfigure}[b]{0.3\textwidth}
        \centering
        \includegraphics[width=\textwidth]{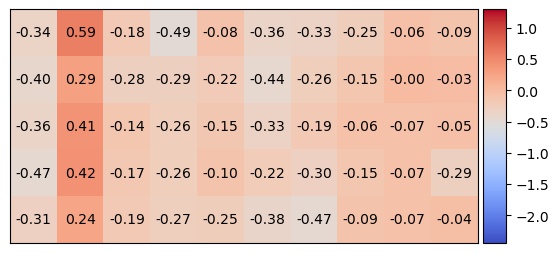}
        \caption{Initialization 0}
        \label{subfig: R0-0}
    \end{subfigure}
    \begin{subfigure}[b]{0.3\textwidth}
        \centering
        \includegraphics[width=\textwidth]{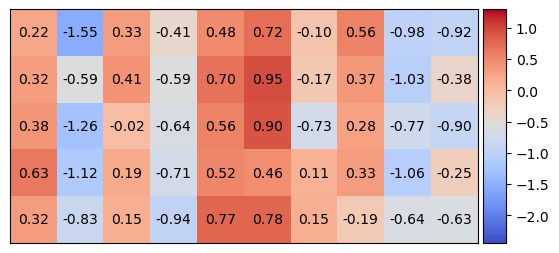}
        \caption{Initialization 1}
        \label{subfig: R0-1}
    \end{subfigure}
    \begin{subfigure}[b]{0.3\textwidth}
        \centering
        \includegraphics[width=\textwidth]{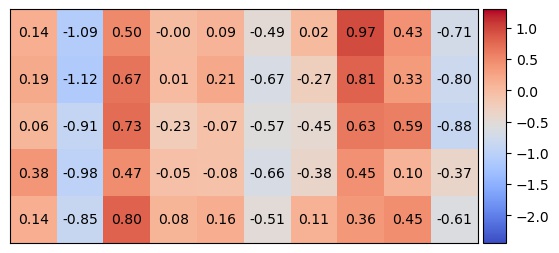}
        \caption{Initialization 2}
        \label{subfig: R0-2}
    \end{subfigure}
    \\
    \begin{subfigure}[b]{0.3\textwidth}
        \centering
        \includegraphics[width=\textwidth]{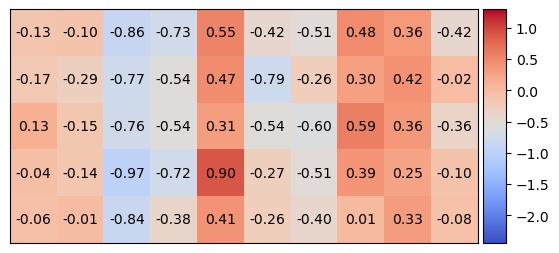}
        \caption{Initialization 3}
        \label{subfig: R0-3}
    \end{subfigure}
    \begin{subfigure}[b]{0.3\textwidth}
        \centering
        \includegraphics[width=\textwidth]{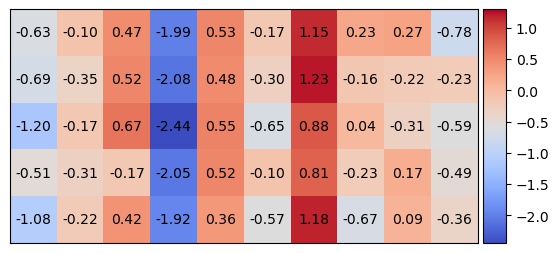}
        \caption{Initialization 4}
        \label{subfig: R0-4}
    \end{subfigure}
    \begin{subfigure}[b]{0.3\textwidth}
        \centering
        \includegraphics[width=\textwidth]{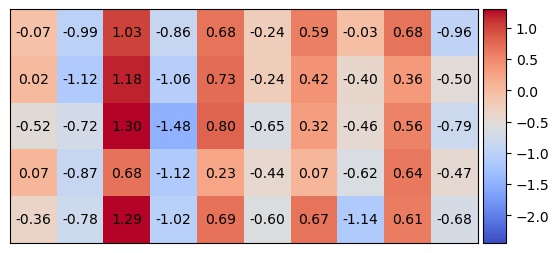}
        \caption{Initialization 5}
        \label{subfig: R0-5}
    \end{subfigure}
    \caption{Learning from one optimal policy, under $\gamma_1$: difference $\hat R - R_{{\rm tr}}$ between learnt and true reward matrices}
    \label{fig: sing-C}
\end{figure}

\begin{figure}[!htp]
    \centering
    \begin{subfigure}[t]{0.45\textwidth}
        \centering
        \includegraphics[width=\textwidth]{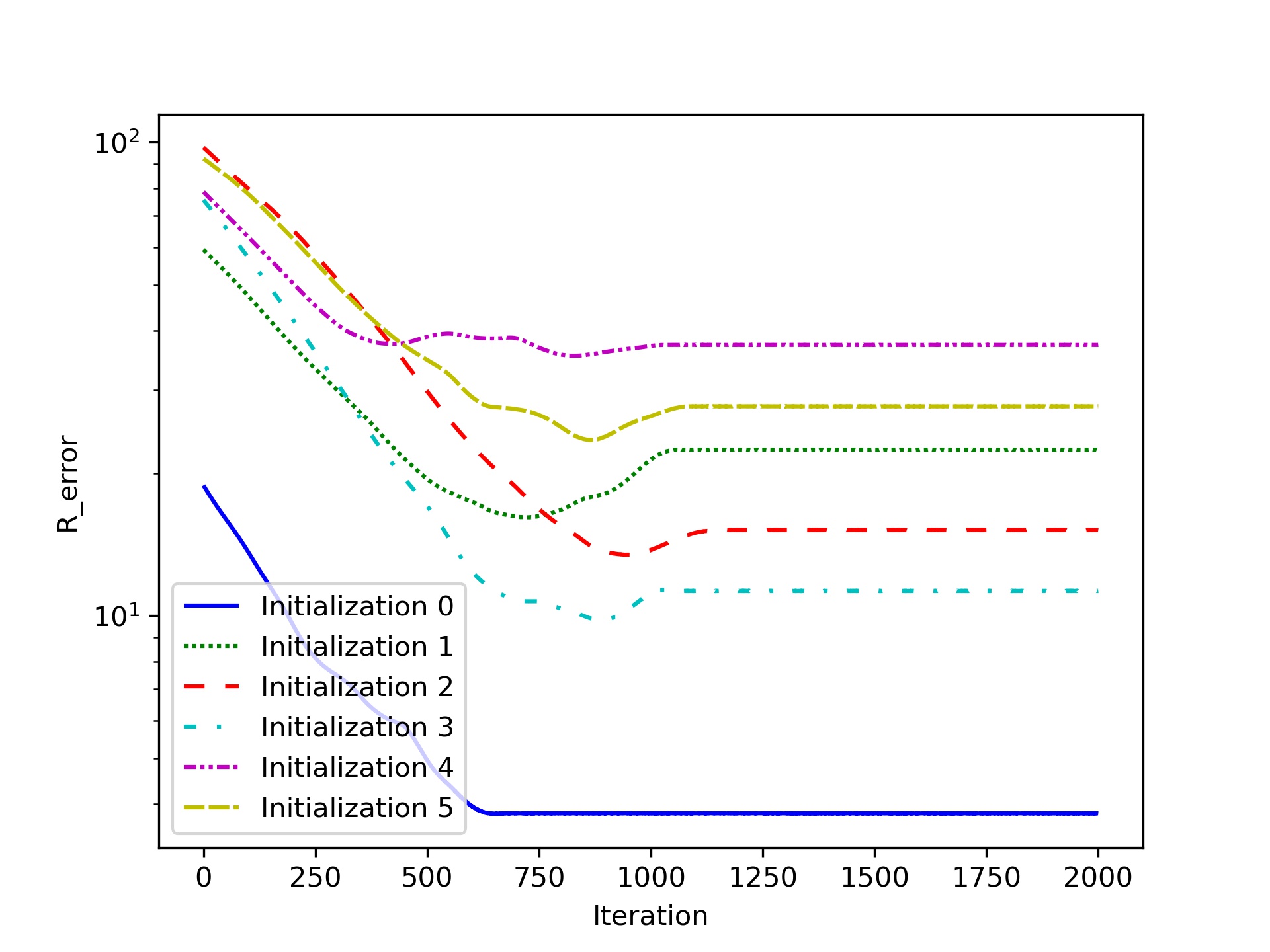}
        \caption{$\ell_2$ error of learnt reward.}
        \label{subfig: sing-comp-c}
    \end{subfigure}
    \begin{subfigure}[t]{0.45\textwidth}
        \centering
        \includegraphics[width=\textwidth]{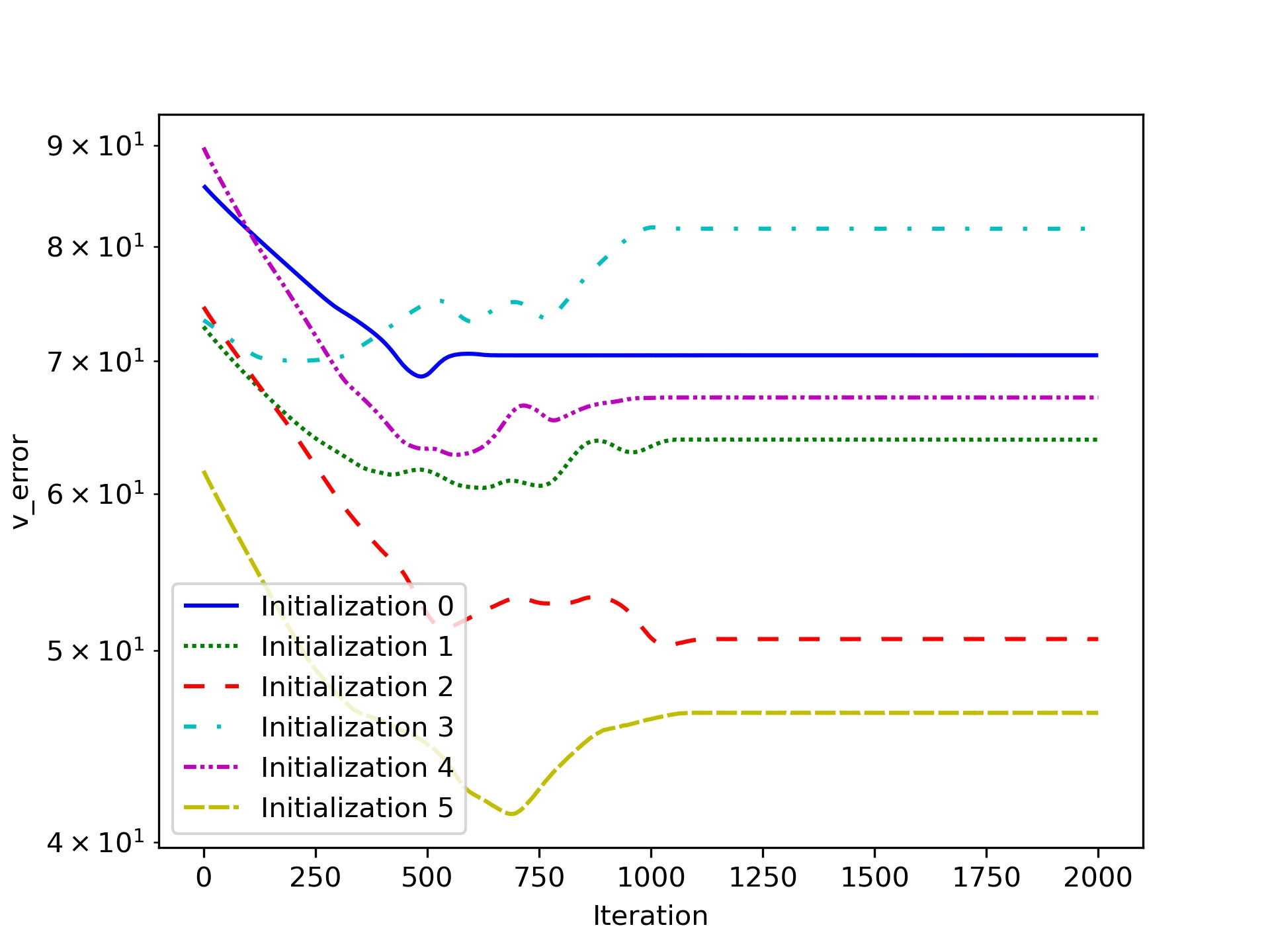}
        \caption{$\ell_2$ error of learnt value function.}
        \label{subfig: sing-comp-v}
    \end{subfigure}
    \caption{Learning from one optimal policy, under $\gamma_1$: comparisons}
    \label{fig: sing-comp}
\end{figure}

\subsection{Uniqueness of IRL with multiple discount rates}
We now demonstrate that the issue of identifiability can be resolved if there is additional information on an optimal policy under the same reward matrix $R_{{\rm tr}}$ but different environment. Here, we assume we are given the  policy $\Pi_1$ optimal with discount factor $\gamma_1=0.95$, and the policy $\Pi_2$ optimal with discount factor $\gamma_2=0.25$. Correspondingly, the loss function for the minimization is adjusted to
\[\begin{aligned}
L_{{\rm doub}}(R, v_1, v_2)&=\frac{1}{2}\sum_{a\in\AC}\sum_{s\in\ST}\left[a^\top \Pi_1s-\exp\left\{a^\top Rs+\gamma_1 s^\top P_a v_1-v_1^\top s\right\}\right]^2\\
&+\frac{1}{2}\sum_{a\in\AC}\sum_{s\in\ST}\left[a^\top \Pi_2s-\exp\left\{a^\top Rs+\gamma_2 s^\top P_a v_2-v_2^\top s\right\}\right]^2.
\end{aligned}\]
An Adam optimizer is adopted with $\alpha=0.005$, $(\beta_1,\beta_2)=(0.5,0.9)$ with overall 2000 minimization steps. With the same set of 6 random initializations for the minimization procedure, the training loss $L_{{\rm doub}}$ also decays rapidly to close to 0. This again suggests that the learnt reward matrix $\tilde R$ can lead to policies $\tilde \Pi_1$ and $\tilde \Pi_2$, each optimal when using the corresponding discount factor $\gamma_1$ and $\gamma_2$, that are close to the given policies $\Pi_1$ and $\Pi_2$; see Figures \ref{fig: doub-char-1} and \ref{fig: doub-char-2}. What differs from the single optimal policy case is that, with the additional information $\Pi_2$, we are able to consistently recover $R_{{\rm tr}}$ up to a constant shift; see Figures \ref{fig: doub-C} and \ref{fig: doub-comp}. Some numerical error remains, due to the optimization algorithm used, as seen by the fact the graphs in  Figure \ref{fig: doub-C} do still vary, and the error in the value function $v_1$ in \ref{fig: doub-comp}(a). Nevertheless, the errors are an order of magnitude less than was observed in Figure \ref{fig: sing-comp} when using observations under a single discount rate.

\begin{figure}[!ht]
    \centering
    \begin{subfigure}[b]{0.3\textwidth}
        \centering
        \includegraphics[width=\textwidth]{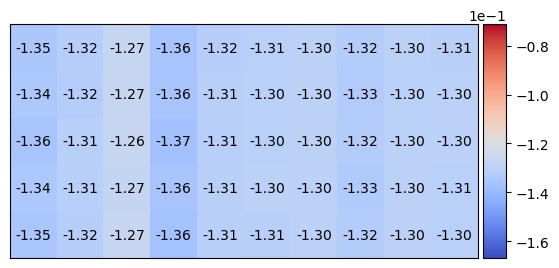}
        \caption{Initialization 0}
        \label{subfig: R1-0}
    \end{subfigure}
    \begin{subfigure}[b]{0.3\textwidth}
        \centering
        \includegraphics[width=\textwidth]{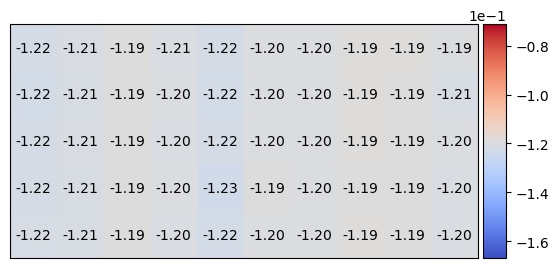}
        \caption{Initialization 1}
        \label{subfig: R1-1}
    \end{subfigure}
    \begin{subfigure}[b]{0.3\textwidth}
        \centering
        \includegraphics[width=\textwidth]{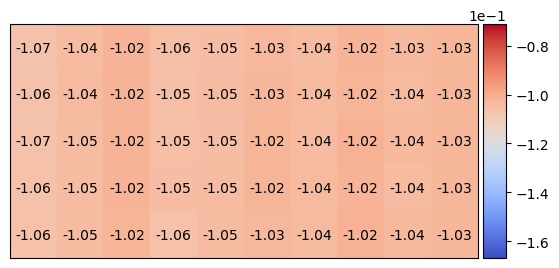}
        \caption{Initialization 2}
        \label{subfig: R1-2}
    \end{subfigure}
    \\
    \begin{subfigure}[b]{0.3\textwidth}
        \centering
        \includegraphics[width=\textwidth]{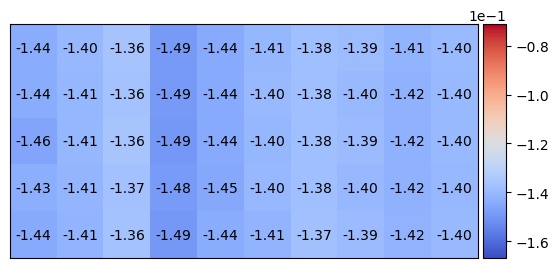}
        \caption{Initialization 3}
        \label{subfig: R1-3}
    \end{subfigure}
    \begin{subfigure}[b]{0.3\textwidth}
        \centering
        \includegraphics[width=\textwidth]{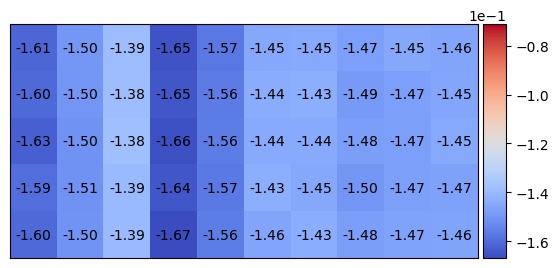}
        \caption{Initialization 4}
        \label{subfig: R1-4}
    \end{subfigure}
    \begin{subfigure}[b]{0.3\textwidth}
        \centering
        \includegraphics[width=\textwidth]{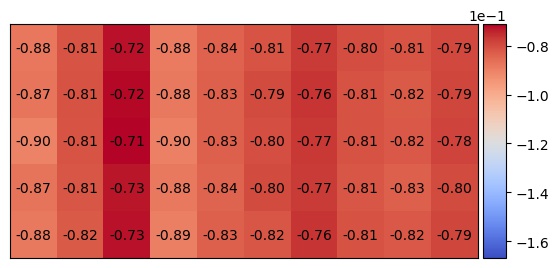}
        \caption{Initialization 5}
        \label{subfig: R1-5}
    \end{subfigure}
    \caption{Learning from two optimal policies, under $\gamma_1$ and $\gamma_2$: difference $\tilde R - R_{\rm tr}$ between learnt and true $R$ matrices. Note scale of $10^{-1}.$}
    \label{fig: doub-C}
\end{figure}

\begin{figure}[!htp]
    \centering
    \begin{subfigure}[t]{0.32\textwidth}
        \centering
        \includegraphics[width=\textwidth]{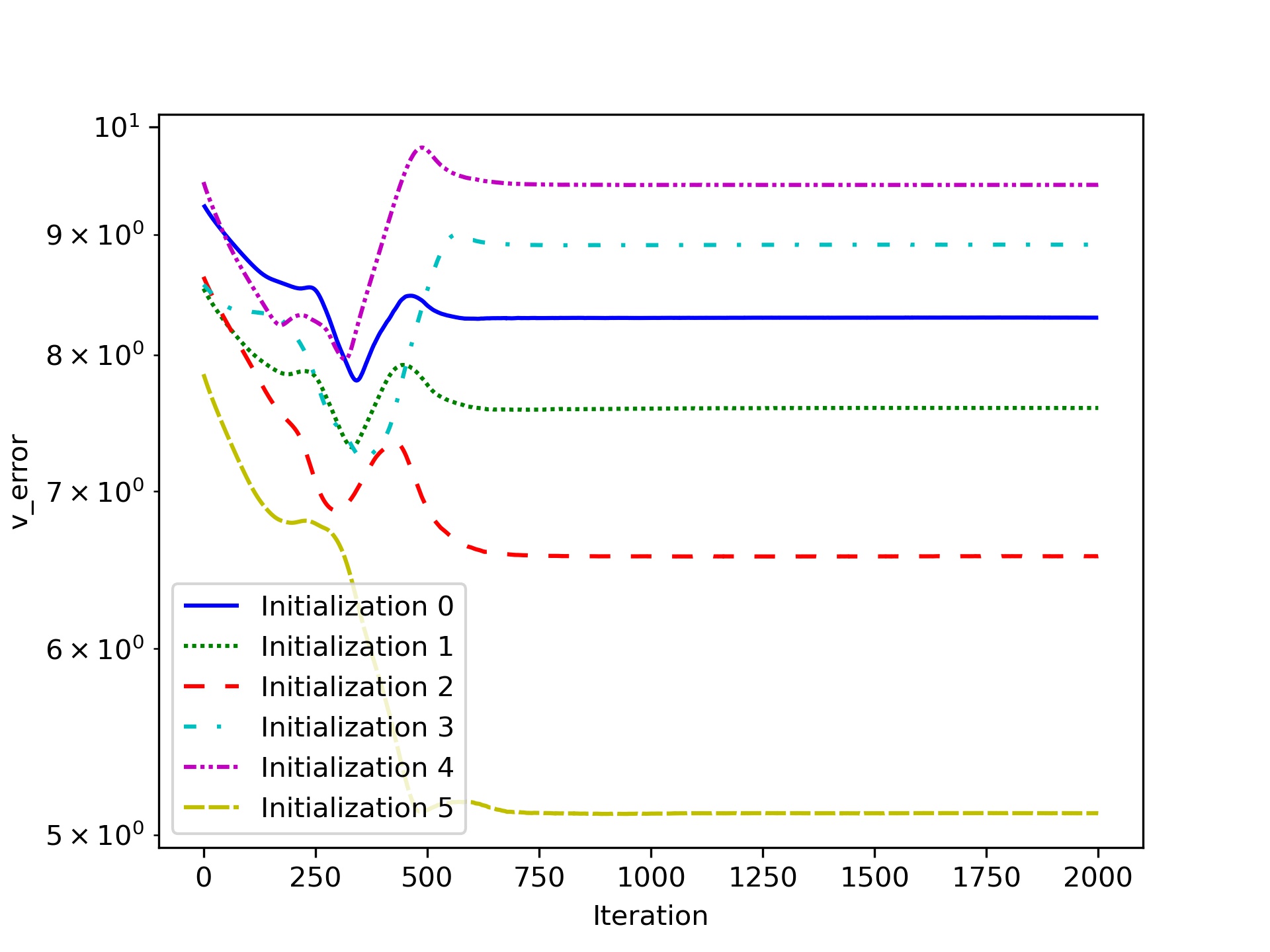}
        \caption{$\ell_2$ error of learnt value\\ function $v_1$ with discount $\gamma_1$}
        \label{subfig: doub-comp-v1}
    \end{subfigure}
    \begin{subfigure}[t]{0.32\textwidth}
        \centering
        \includegraphics[width=\textwidth]{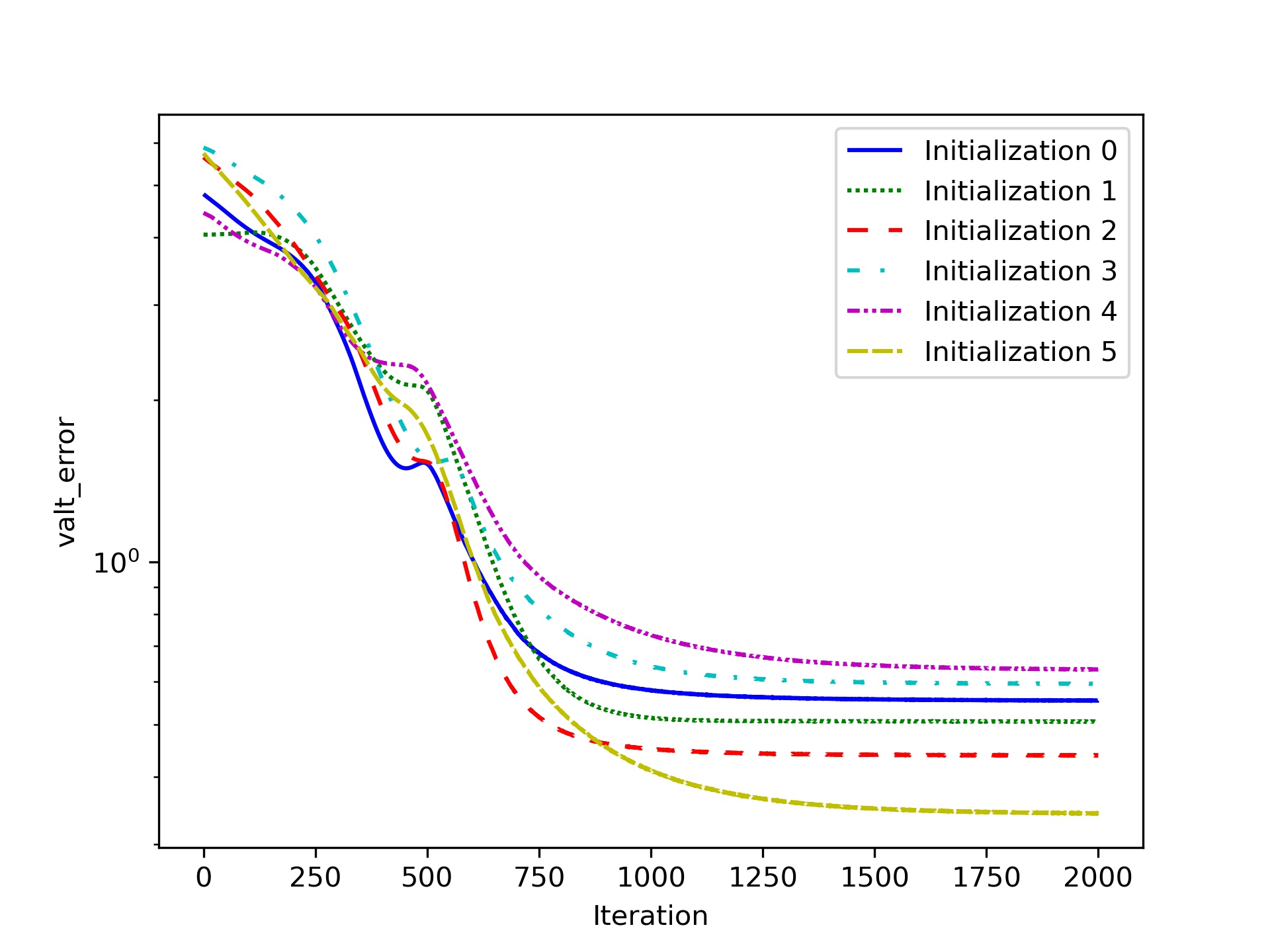}
        \caption{$\ell_2$ error of learnt value\\ function $v_2$  with discount  $\gamma_2$}
        \label{subfig: doub-comp-v2}
    \end{subfigure}
    \begin{subfigure}[t]{0.32\textwidth}
        \centering
        \includegraphics[width=\textwidth]{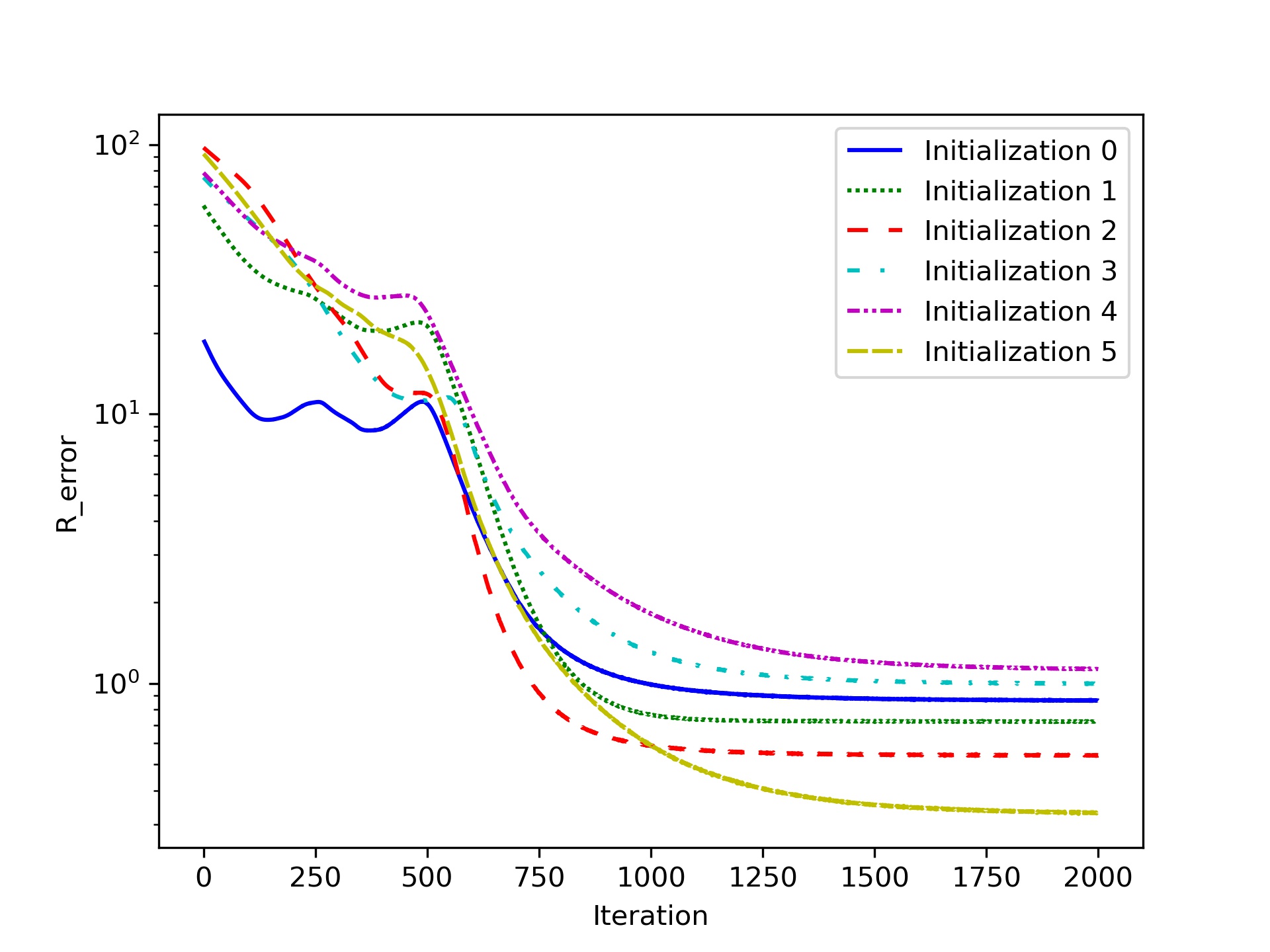}
        \caption{$\ell_2$ error of learnt reward matrix}
    \end{subfigure}
    \caption{Two optimal policies under $\gamma_1$ and $\gamma_2$: comparisons}
    \label{fig: doub-comp}
\end{figure}
\begin{figure}[!htp]
    \centering
    \begin{subfigure}[b]{0.3\textwidth}
        \centering
        \includegraphics[width=\textwidth]{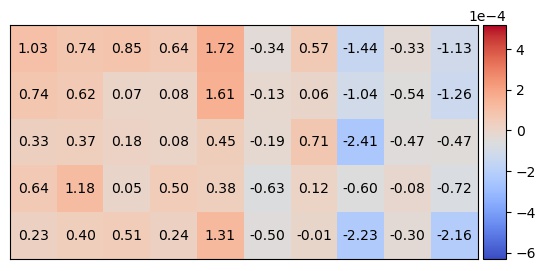}
        \caption{Initialization 0}
        \label{subfig: csst-0}
    \end{subfigure}
    \begin{subfigure}[b]{0.3\textwidth}
        \centering
        \includegraphics[width=\textwidth]{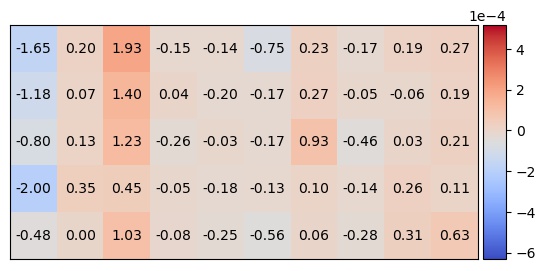}
        \caption{Initialization 1}
        \label{subfig: csst-1}
    \end{subfigure}
    \begin{subfigure}[b]{0.3\textwidth}
        \centering
        \includegraphics[width=\textwidth]{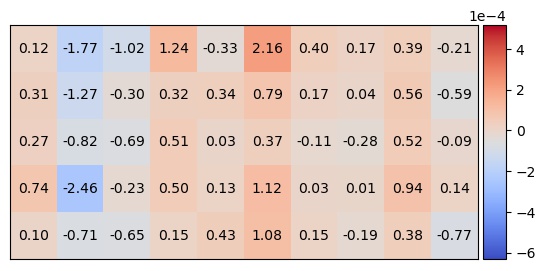}
        \caption{Initialization 2}
        \label{subfig: csst-2}
    \end{subfigure}
    \\
    \begin{subfigure}[b]{0.3\textwidth}
        \centering
        \includegraphics[width=\textwidth]{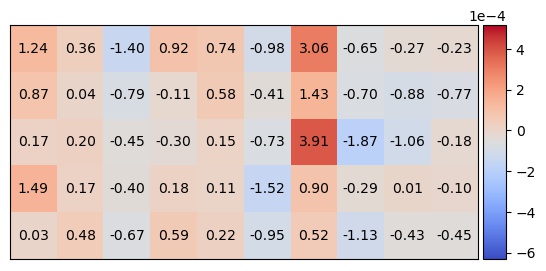}
        \caption{Initialization 3}
        \label{subfig: csst-3}
    \end{subfigure}
    \begin{subfigure}[b]{0.3\textwidth}
        \centering
        \includegraphics[width=\textwidth]{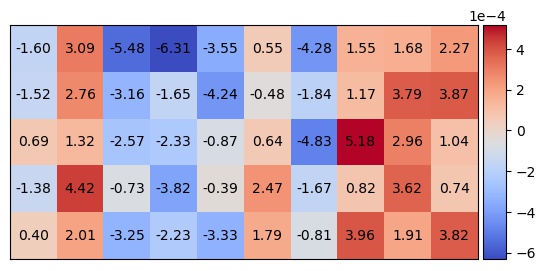}
        \caption{Initialization 4}
        \label{subfig: csst-4}
    \end{subfigure}
    \begin{subfigure}[b]{0.3\textwidth}
        \centering
        \includegraphics[width=\textwidth]{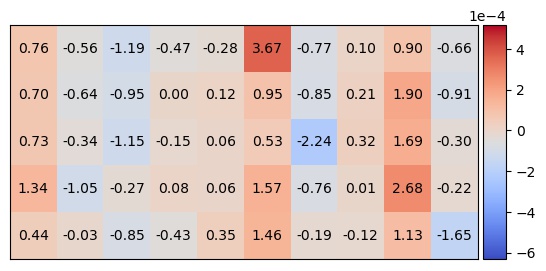}
        \caption{Initialization 5}
        \label{subfig: csst-5}
    \end{subfigure}
    \caption{Learning from optimal policy under $\gamma_1$: difference $\hat\Pi_1-\Pi_1$ between optimal policy under the learnt model and the true optimal policy. Note scale of $10^{-4}$.}
    \label{fig: sing-char}
\end{figure}
\begin{figure}[!htp]
    \centering
    \begin{subfigure}[b]{0.3\textwidth}
        \centering
        \includegraphics[width=\textwidth]{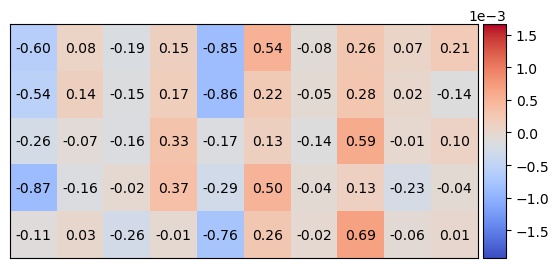}
        \caption{Initialization 0}
        \label{subfig: csst-1-0}
    \end{subfigure}
    \begin{subfigure}[b]{0.3\textwidth}
        \centering
        \includegraphics[width=\textwidth]{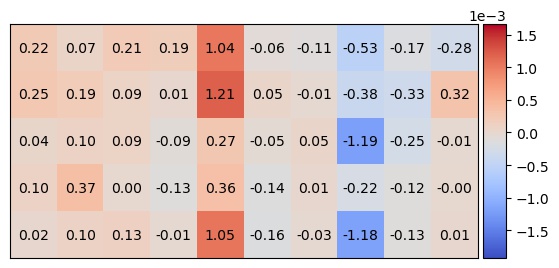}
        \caption{Initialization 1}
        \label{subfig: csst-1-1}
    \end{subfigure}
    \begin{subfigure}[b]{0.3\textwidth}
        \centering
        \includegraphics[width=\textwidth]{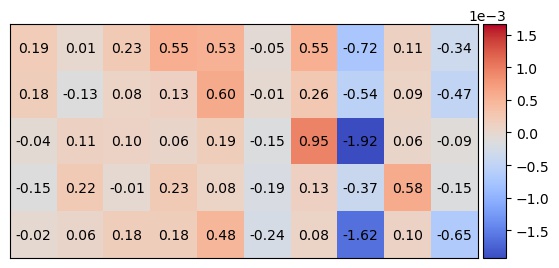}
        \caption{Initialization 2}
        \label{subfig: csst-1-2}
    \end{subfigure}
    \\
    \begin{subfigure}[b]{0.3\textwidth}
        \centering
        \includegraphics[width=\textwidth]{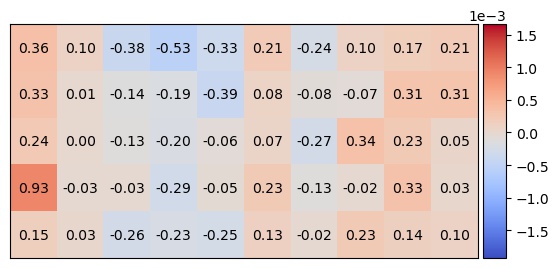}
        \caption{Initialization 3}
        \label{subfig: csst-1-3}
    \end{subfigure}
    \begin{subfigure}[b]{0.3\textwidth}
        \centering
        \includegraphics[width=\textwidth]{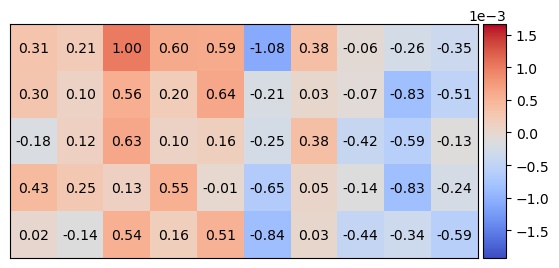}
        \caption{Initialization 4}
        \label{subfig: csst-1-4}
    \end{subfigure}
    \begin{subfigure}[b]{0.3\textwidth}
        \centering
        \includegraphics[width=\textwidth]{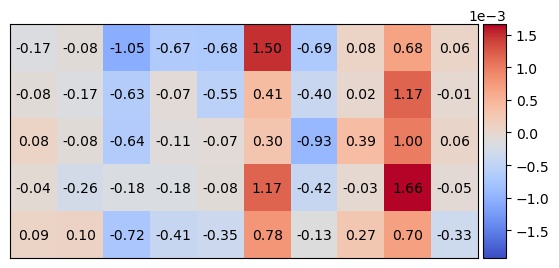}
        \caption{Initialization 5}
        \label{subfig: csst-1-5}
    \end{subfigure}
    \caption{Learning from  optimal policies under $\gamma_1$ and $\gamma_2$: difference  $\tilde\Pi_1 - \Pi_1$ between learnt and true policies under $\gamma_1$. Note scale of $10^{-3}$.}
    \label{fig: doub-char-1}
\end{figure}
\begin{figure}[!htp]
    \centering
    \begin{subfigure}[b]{0.3\textwidth}
        \centering
        \includegraphics[width=\textwidth]{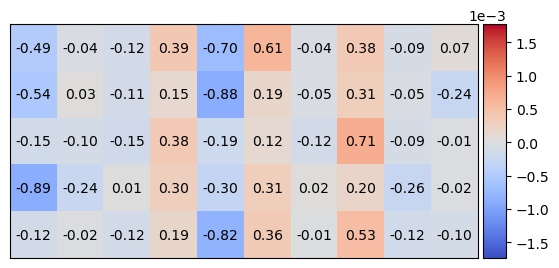}
        \caption{Initialization 0}
        \label{subfig: csst-2-0}
    \end{subfigure}
    \begin{subfigure}[b]{0.3\textwidth}
        \centering
        \includegraphics[width=\textwidth]{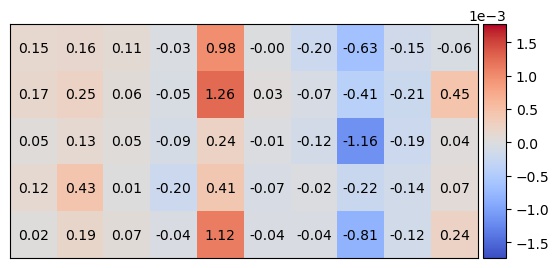}
        \caption{Initialization 1}
        \label{subfig: csst-2-1}
    \end{subfigure}
    \begin{subfigure}[b]{0.3\textwidth}
        \centering
        \includegraphics[width=\textwidth]{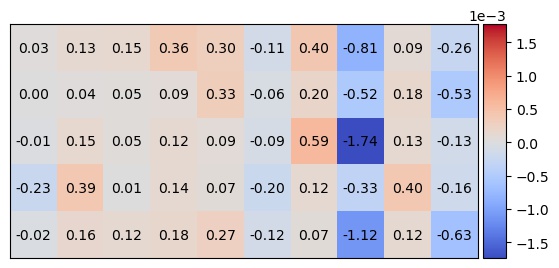}
        \caption{Initialization 2}
        \label{subfig: csst-2-2}
    \end{subfigure}
    \\
    \begin{subfigure}[b]{0.3\textwidth}
        \centering
        \includegraphics[width=\textwidth]{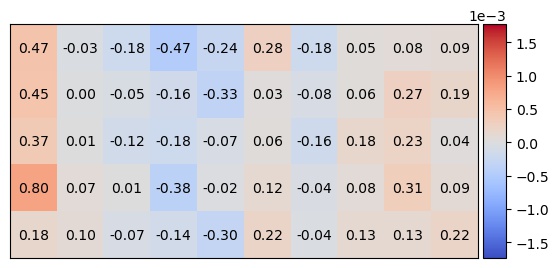}
        \caption{Initialization 3}
        \label{subfig: csst-2-3}
    \end{subfigure}
    \begin{subfigure}[b]{0.3\textwidth}
        \centering
        \includegraphics[width=\textwidth]{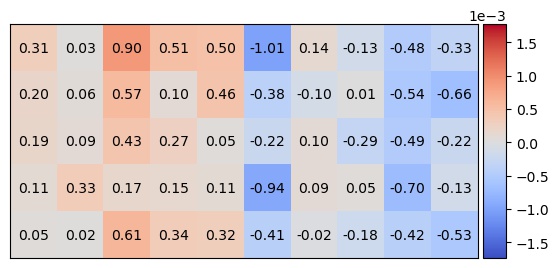}
        \caption{Initialization 4}
        \label{subfig: csst-2-4}
    \end{subfigure}
    \begin{subfigure}[b]{0.3\textwidth}
        \centering
        \includegraphics[width=\textwidth]{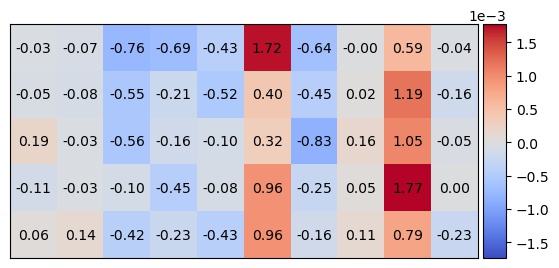}
        \caption{Initialization 5}
        \label{subfig: csst-2-5}
    \end{subfigure}
    \caption{Learning from optimal policies under $\gamma_1$ and $\gamma_2$: differences $\tilde\Pi_2 - \Pi_2$ between learnt and true policies under $\gamma_2$. Note scale of $10^{-3}$.}
    \label{fig: doub-char-2}
\end{figure}

\section*{Acknowledgements}
The authors acknowledge the support of the Alan Turing
Institute under the Engineering and Physical Sciences Research Council grant EP/N510129/1. Samuel Cohen also acknowledges the support of the Oxford-Man Institute
for Quantitative Finance. As a visiting scholar, Haoyang Cao also appreciates the support provided by the Mathematical Institute at University of Oxford.

\bibliographystyle{plainnat}
\bibliography{irl}

\begin{thebibliography}{28}
\providecommand{\natexlab}[1]{#1}
\providecommand{\url}[1]{\texttt{#1}}
\expandafter\ifx\csname urlstyle\endcsname\relax
  \providecommand{\doi}[1]{doi: #1}\else
  \providecommand{\doi}{doi: \begingroup \urlstyle{rm}\Url}\fi

\bibitem[Abbeel and Ng(2004)]{abbeel2004apprenticeship}
Pieter Abbeel and Andrew~Y Ng.
\newblock Apprenticeship learning via inverse reinforcement learning.
\newblock In \emph{Proceedings of the twenty-first international conference on
  Machine learning}, page~1, 2004.

\bibitem[Amin and Singh(2016)]{amin2016towards}
Kareem Amin and Satinder Singh.
\newblock Towards resolving unidentifiability in inverse reinforcement
  learning.
\newblock \emph{arXiv preprint arXiv:1601.06569}, 2016.

\bibitem[Amin et~al.(2017)Amin, Jiang, and Singh]{amin2017repeated}
Kareem Amin, Nan Jiang, and Satinder Singh.
\newblock Repeated inverse reinforcement learning.
\newblock \emph{Advances in Neural Information Processing Systems},
  30:\penalty0 1815--1824, 2017.

\bibitem[Balakrishnan et~al.(2020)Balakrishnan, Nguyen, Low, and
  Soh]{balakrishnan2020efficient}
Sreejith Balakrishnan, Quoc~Phong Nguyen, Bryan Kian~Hsiang Low, and Harold
  Soh.
\newblock Efficient exploration of reward functions in inverse reinforcement
  learning via bayesian optimization.
\newblock \emph{Advances in Neural Information Processing Systems}, 33, 2020.

\bibitem[Bertsekas and Shreve(2004)]{bertsekas2004stochastic}
Dimitir~P Bertsekas and Steven Shreve.
\newblock \emph{Stochastic optimal control: the discrete-time case}.
\newblock 2004.

\bibitem[Boularias et~al.(2011)Boularias, Kober, and
  Peters]{boularias2011relative}
Abdeslam Boularias, Jens Kober, and Jan Peters.
\newblock Relative entropy inverse reinforcement learning.
\newblock In \emph{Proceedings of the Fourteenth International Conference on
  Artificial Intelligence and Statistics}, pages 182--189. JMLR Workshop and
  Conference Proceedings, 2011.

\bibitem[Boyd et~al.(1994)Boyd, El~Ghaoui, Feron, and
  Balakrishnan]{boyd1994linear}
Stephen Boyd, Laurent El~Ghaoui, Eric Feron, and Venkataramanan Balakrishnan.
\newblock \emph{Linear matrix inequalities in system and control theory}.
\newblock SIAM, 1994.

\bibitem[Dupuis and Ellis(2011)]{dupuis2011weak}
Paul Dupuis and Richard~S Ellis.
\newblock \emph{A weak convergence approach to the theory of large deviations},
  volume 902.
\newblock John Wiley \& Sons, 2011.

\bibitem[Dvijotham and Todorov(2010)]{dvijotham2010}
Krishnamurthy Dvijotham and Emanuel Todorov.
\newblock Inverse optimal control with linearly-solvable mdps.
\newblock In \emph{Proceedings of the 27th International Conference on
  International Conference on Machine Learning}, ICML'10, page 335–342,
  Madison, WI, USA, 2010. Omnipress.
\newblock ISBN 9781605589077.

\bibitem[Finn et~al.(2016{\natexlab{a}})Finn, Christiano, Abbeel, and
  Levine]{finn2016connection}
Chelsea Finn, Paul Christiano, Pieter Abbeel, and Sergey Levine.
\newblock A connection between generative adversarial networks, inverse
  reinforcement learning, and energy-based models.
\newblock \emph{arXiv preprint arXiv:1611.03852}, 2016{\natexlab{a}}.

\bibitem[Finn et~al.(2016{\natexlab{b}})Finn, Levine, and
  Abbeel]{finn2016guided}
Chelsea Finn, Sergey Levine, and Pieter Abbeel.
\newblock Guided cost learning: Deep inverse optimal control via policy
  optimization.
\newblock In \emph{International conference on machine learning}, pages 49--58,
  2016{\natexlab{b}}.

\bibitem[Fu et~al.(2018)Fu, Luo, and Levine]{fu2017learning}
Justin Fu, Katie Luo, and Sergey Levine.
\newblock Learning robust rewards with adverserial inverse reinforcement
  learning.
\newblock In \emph{International Conference on Learning Representations}, 2018.

\bibitem[Haarnoja et~al.(2017)Haarnoja, Tang, Abbeel, and
  Levine]{haarnoja2017reinforcement}
Tuomas Haarnoja, Haoran Tang, Pieter Abbeel, and Sergey Levine.
\newblock Reinforcement learning with deep energy-based policies.
\newblock In \emph{International Conference on Machine Learning}, pages
  1352--1361. PMLR, 2017.

\bibitem[Kalman(1964)]{Kalman64}
R.~E. Kalman.
\newblock {When Is a Linear Control System Optimal?}
\newblock \emph{Journal of Basic Engineering}, 86\penalty0 (1):\penalty0
  51--60, 03 1964.

\bibitem[Keeney and Raiffa(1976)]{keeney_raiffa_1993}
Ralph~L. Keeney and Howard Raiffa.
\newblock \emph{Decisions with Multiple Objectives: Preferences and Value
  Trade-Offs}.
\newblock Wiley, 1976.

\bibitem[Kim et~al.(2021)Kim, Garg, Shiragur, and Ermon]{Kim2021}
Kuno Kim, Shivam Garg, Kirankumar Shiragur, and Stefano Ermon.
\newblock Reward identification in inverse reinforcement learning.
\newblock In Marina Meila and Tong Zhang, editors, \emph{Proceedings of the
  38th International Conference on Machine Learning}, volume 139 of
  \emph{Proceedings of Machine Learning Research}, pages 5496--5505. PMLR,
  18--24 Jul 2021.
\newblock URL \url{http://proceedings.mlr.press/v139/kim21c.html}.

\bibitem[Levine(2018)]{levine2018reinforcement}
Sergey Levine.
\newblock Reinforcement learning and control as probabilistic inference:
  Tutorial and review.
\newblock \emph{arXiv preprint arXiv:1805.00909}, 2018.

\bibitem[Levine et~al.(2011)Levine, Popovic, and Koltun]{levine2011nonlinear}
Sergey Levine, Zoran Popovic, and Vladlen Koltun.
\newblock Nonlinear inverse reinforcement learning with gaussian processes.
\newblock \emph{Advances in neural information processing systems},
  24:\penalty0 19--27, 2011.

\bibitem[Lucas(1976)]{Lucas1976}
Robert~E. Lucas.
\newblock Econometric policy evaluation: A critique.
\newblock \emph{Carnegie-Rochester Conference Series on Public Policy},
  1:\penalty0 19--46, 1976.

\bibitem[Ng and Russell(2000)]{ng2000algorithms}
Andrew~Y Ng and Stuart Russell.
\newblock Algorithms for inverse reinforcement learning.
\newblock In \emph{Proceedings of Seventeenth International Conference on
  Machine Learning}. Citeseer, 2000.

\bibitem[Ng et~al.(1999)Ng, Harada, and Russell]{Ng99policyinvariance}
Andrew~Y. Ng, Daishi Harada, and Stuart Russell.
\newblock Policy invariance under reward transformations: Theory and
  application to reward shaping.
\newblock In \emph{In Proceedings of the Sixteenth International Conference on
  Machine Learning}, pages 278--287. Morgan Kaufmann, 1999.

\bibitem[Puterman(2014)]{puterman2014markov}
Martin~L Puterman.
\newblock \emph{Markov decision processes: discrete stochastic dynamic
  programming}.
\newblock John Wiley \& Sons, 2014.

\bibitem[Ratliff et~al.(2006)Ratliff, Bagnell, and Zinkevich]{ratliff2006}
Nathan~D. Ratliff, J.~Andrew Bagnell, and Martin~A. Zinkevich.
\newblock Maximum margin planning.
\newblock In \emph{Proceedings of the 23rd International Conference on Machine
  Learning}, ICML '06, page 729–736, New York, NY, USA, 2006. Association for
  Computing Machinery.
\newblock ISBN 1595933832.
\newblock \doi{10.1145/1143844.1143936}.
\newblock URL \url{https://doi.org/10.1145/1143844.1143936}.

\bibitem[Russell(1998)]{russell1998learning}
Stuart Russell.
\newblock Learning agents for uncertain environments.
\newblock In \emph{Proceedings of the Eleventh Annual Conference on
  Computational Learning Theory}, pages 101--103, 1998.

\bibitem[Sargent(1978)]{sargent1978estimation}
Thomas~J Sargent.
\newblock Estimation of dynamic labor demand schedules under rational
  expectations.
\newblock \emph{Journal of Political Economy}, 86\penalty0 (6):\penalty0
  1009--1044, 1978.

\bibitem[Seneta(2006)]{Seneta2006}
E.~Seneta.
\newblock \emph{Non-negative Matrices and {M}arkov chains}.
\newblock Springer, revised printing edition, 2006.

\bibitem[Ziebart(2010)]{ziebart2010modeling}
Brian~D Ziebart.
\newblock \emph{{Modeling Purposeful Adaptive Behavior with the Principle of
  Maximum Causal Entropy}}.
\newblock PhD thesis, Carnegie Mellon University, 2010.

\bibitem[Ziebart et~al.(2008)Ziebart, Maas, Bagnell, and
  Dey]{ziebart2008maximum}
Brian~D Ziebart, Andrew~L Maas, J~Andrew Bagnell, and Anind~K Dey.
\newblock Maximum entropy inverse reinforcement learning.
\newblock In \emph{Aaai}, volume~8, pages 1433--1438. Chicago, IL, USA, 2008.

\end{thebibliography}

\newpage
\appendix

\section*{Appendix: A discussion of guided cost learning and related maximum entropy inverse reinforcement learning models}\label{sec:GCL}

The guided cost learning algorithm was proposed in \citet{finn2016guided} to solve an (undiscounted) inverse reinforcement learning problem over a finite time horizon with a finite state-action space $(\ST,\AC)$. In \citet{finn2016guided}, instead of directly modelling the optimal feedback policy, the optimal trajectory distribution is taken as the starting point for inference. Adopting the idea of the maximum casual entropy model in \citet{ziebart2010modeling} (phrased in terms of rewards rather than costs) a common interpretation of the algorithm assumes we observe trajectories $\tau$ sampled from the distribution \begin{equation}\label{eq:gcl_pf}
p^f(\tau=(s_0^\tau,a_0^\tau,\dots,s_{T-1}^\tau,a_{T-1}^\tau,s_T^\tau))=\frac{1}{Z^f}\exp\bigg\{\sum_{t=0}^{T-1}f(s_t^\tau, a_t^\tau)\bigg\},
\end{equation}
where the partition factor \[Z^f=\sum_{\tau}\exp\bigg\{\sum_{t=0}^{T-1}f(s_t^\tau, a_t^\tau)\bigg\}=\E_{\tau\sim q}\bigg[\exp\bigg\{\sum_{t=0}^{T-1}f(s_t^\tau, a_t^\tau)\bigg\}\bigg/q(\tau)\bigg]\]
is estimated through importance sampling with the `ambient distribution' $q(\tau)$, which can be chosen arbitrarily\footnote{
An additional complexity in the guided cost learning algorithm is that the reward function and the ambient distribution are updated iteratively. Numerically, this can be seen as a variance reduction technique, rather than a conceptual change to the algorithm. First, the reward function $f$ is updated by alternately maximizing the log likelihood $\log p^f(\tau)$ over the demonstrator's trajectories $\{\tau^*_i\}_{i=1}^N$, which is equivalent to solving $\hat f=\argmin_fD_{\rm{KL}}(q^*\|p^f)$. Secondly, the ambient distribution $q$ is updated by minimizing the KL divergence $D_{KL}(q\|p^f)$ using the trajectories $\{\tau^q_j\}_{j=1}^M$ sampled from $q(\tau)=\mu_0(s^\tau_0)\prod_{t=0}^{T-1}\pi_t(a^\tau_t|s^\tau_t)\transprob(s^\tau_{t+1}|s^\tau_t,a^\tau_t)$. Using this method, the transition probabilities $\transprob$ can also be estimated, and $q$ can be seen as closely related to the law $\bar p^f$ in \eqref{eq:gcl_p2f}.}. 

As mentioned above, and discussed further by \citet{ziebart2008maximum} and \citet{levine2018reinforcement}, this is consistent with our entropy regularized MDP when transitions are deterministic, but differs for stochastic problems. An alternative maximum entropy model, which incorporates knowledge of $\transprob$, assumes trajectories  are sampled from
\begin{equation}\label{eq:gcl_p2f}
\bar p^f(\tau)=\frac{\mu_0(s^\tau_0)}{Z^f}\prod_{t=0}^{T-1}\exp\Big\{f(s^\tau_t,a^\tau_t)\Big\}\transprob(s^\tau_{t+1}|s^\tau_t,a^\tau_t).
\end{equation}

To see how this connects to the entropy regularized MDP, we observe that a entropy-regularized optimizing agent will generate trajectories with distribution 
\begin{equation}\label{eq:gcl_q}
q^*(\tau)=\mu_0(s^\tau_0)\prod_{t=0}^{T-1}\pi^*_t(a^\tau_t|s^\tau_t)\transprob(s^\tau_{t+1}|s^\tau_t,a^\tau_t),\
\end{equation}
where $\pi^*=\{\pi^*_t\}_{t=0}^{T-1}$ solves the problem discussed in Section \ref{sec:finitehorizon}.

Given that we do not have an infinite-horizon time-homogenous system, the optimal policy $\pi^*$ is typically time-dependent and this is reflected in the density $q^*$, and hence in the trajectories we observe. Using $\bar p^f$ in \eqref{eq:gcl_p2f} as the basis of the guided cost learning algorithm, the demonstrator's optimal trajectory distribution $q^*$ can be written in the desired form (i.e. for some choice of $f$ in \eqref{eq:gcl_p2f}, which may or may not correspond to the agent's rewards), provided the underlying $f_{\mathrm{true}}$ and $g_{\mathrm{true}}$ lead to a time-invariant optimal policy $\pi^*$. Otherwise, one should further adjust the guided cost learning model $\bar p^f$ to include time-dependent rewards $f$, that is,
\begin{equation}\label{eq:gcl_p3f}
\tilde p^f(\tau)=\frac{\mu_0(s^\tau_0)}{Z^f}\prod_{t=0}^{T-1}\exp\Big\{f(t,s^\tau_t,a^\tau_t)\Big\}\transprob(s^\tau_{t+1}|s^\tau_t,a^\tau_t).
\end{equation}

With the addition of time-dependent rewards, it is interesting to consider what this variation of guided cost learning will output. Suppose we observe a large number of trajectories and estimate a reward $f_{\mathrm{est}}$ to maximize the likelihood \eqref{eq:gcl_p3f}, or equivalently to minimize the KL divergence $D_{\mathrm{KL}}(q^*\|\tilde p^f)$. Comparing $\tilde p^f$ in \eqref{eq:gcl_p3f} with $q^*$ in \eqref{eq:gcl_q}, we see that the minimum KL divergence is $D_{\mathrm{KL}}(q^*\|\tilde p^f)=0$, which is achieved when, for each $t\in \{0\dots,T-1\},$
\begin{align*}
    f_{\mathrm{est}}(t,s,a)+c_t &= Q^*_t(s,a) - V^*_t(s) \\
    &= f_{\mathrm{true}}(t, s,a)+\E_{S'\sim\transprob(\cdot|s,a)}\left[V^*_{t+1}(S')\right] - V^*_t(s) 
\end{align*}
where $c_t\in \rn$ is a constant (which may depend on $t$, but not on $s$). This will yield $f_{\mathrm{true}}= f_{\mathrm{est}}$ provided $(t,s)\mapsto V^*_t(s)$ is a deterministic function of time (i.e. it is independent of $s$), and this is a necessary condition for nontrivial $\transprob$.

In other words, the identifiability issue discussed in the main body of this paper remains, as the demonstrator's trajectory distribution will depend on the state-action value function $Q_t^*$, rather than directly on the reward. Furthermore, this variation of guided cost learning generally corresponds to finding a reward which generates the observed policy, \emph{and yields a value function $V^*$ which does not vary with the state of the system}. Of course, this reward will not usually be the same as that faced by the demonstrator, and so the results of guided cost learning are not guaranteed to generalize to agents with different transition probabilities.

We note that \cite{balakrishnan2020efficient} discuss the non-identifiability of costs in a MaxEntIRL approach. Their work focuses on building a projection under which rewards resulting in similar policies are mapped together, and then build a Bayesian estimation method for this projected data. What we have seen is that this approach is consistent (after the modifications discussed above), and will identify \emph{some} cost function which gives the corresponding policy. For the entropy-regularized problem, our results precisely describe the kernel of this projection -- it must correspond to different choices of the value function for the system.

\end{document}